%% file: main.tex
\documentclass[11pt,oneside]{jmlr}
\makeatletter
\def\ps@jmlrtps{%
  \let\@mkboth\@gobbletwo
  \def\@oddhead{}\let\@evenhead\@oddhead
  \def\@oddfoot{\hfill\small\rmfamily\thepage\hfill}\let\@evenfoot\@oddfoot
}
\def\ps@jmlrps{%
  \let\@mkboth\@gobbletwo
  \def\@oddhead{}\let\@evenhead\@oddhead
  \def\@oddfoot{\hfill\small\rmfamily\thepage\hfill}\let\@evenfoot\@oddfoot
}
\makeatother
\renewenvironment{proof}[1][\proofname]{\par\noindent{\bfseries\upshape #1\ }}{\jmlrQED}
\firstpageno{1}

\usepackage{booktabs,array}
\usepackage{enumitem}
\usepackage{mathtools}
\usepackage{microtype}
\usepackage{pgfplots}
\pgfplotsset{compat=1.18}
\usepgfplotslibrary{groupplots}
\SetAlgoVlined
\numberwithin{equation}{section}
\numberwithin{theorem}{section}

\newcommand{\E}{\mathbb E}
\newcommand{\R}{\mathbb R}
\newcommand{\X}{\mathcal X}
\newcommand{\F}{\mathcal F}
\newcommand{\norm}[1]{\left\lVert#1\right\rVert}
\newcommand{\inner}[1]{\left\langle#1\right\rangle}
\newcommand{\Rate}{\mathfrak R_T}
\DeclareMathOperator{\KL}{KL}
\DeclareMathOperator{\sgn}{sgn}
\allowdisplaybreaks[1]
\title{Parameter-Free Dynamic Regret under Heavy-Tailed Noise}
\author{%
  \Name{Vaneet Aggarwal}\\
  \addr Purdue University, West Lafayette, Indiana, USA
}
\hypersetup{
  pdftitle={Parameter-Free Dynamic Regret under Heavy-Tailed Noise},
  pdfauthor={Vaneet Aggarwal}
}

\begin{document}
\maketitle

\begin{abstract}
We study online convex optimization on a compact convex domain of diameter $D$ in which the learner observes, at each round, one stochastic subgradient and no loss values. The observations are conditionally unbiased, and the noise has an unconditional $p$-th moment at most $\sigma^p$ for an unknown $p\in(1,2]$; for $p<2$ its variance may be infinite. For a comparator sequence of path length $P_T$ we write $\Lambda_T=1+P_T/D$, and we show that the minimax expected dynamic regret is, up to numerical constants,
\[
\min\Bigl\{GD\sqrt{T\Lambda_T}+\sigma DT^{1/p}\Lambda_T^{(p-1)/p},\ GDT\Bigr\},
\]
where $G$ bounds the subgradients. The lower bound holds on every compact convex domain of positive diameter. An algorithm that combines AdaGrad learners restarted at different frequencies through exponential weights with a data-dependent learning rate attains this rate for each fixed $p>1$, without knowing $G$, $\sigma$, $p$, or $P_T$: the constant of the first term is numerical, the constant of the second is $O(1+\log(p/(p-1)))$, and no logarithmic factor of $T$ appears. The analysis bounds the gradient energy separately on each restart block, which gives the exponent $(p-1)/p$ of $\Lambda_T$ in the noise term; a bound over the whole horizon gives the exponent $1/2$. Running the algorithm on epochs of doubling length removes the need to know $T$ at the cost of a factor $2+\sqrt2$. The appendices give a high-probability version and a synthetic study.
\end{abstract}

\input{sections/01_introduction}
\input{sections/02_related_work}
\input{sections/03_setup}
\input{sections/04_algorithm}
\input{sections/05_upper_bounds}
\input{sections/06_lower_bounds}
\input{sections/07_conclusion}

\section*{AI Use Statement}
Generative AI tools were used to assist with language editing, document restructuring, LaTeX formatting, and consistency checks. The author reviewed all AI-assisted work and takes responsibility for the final text, claims, proofs, citations, and artifacts.

\appendix
\input{sections/A_one_projection}
\input{sections/B_auxiliary_proofs}
\input{sections/C_master_proof}
\input{sections/D_lower_det_proof}
\input{sections/E_anytime_proof}
\input{sections/F_high_probability}
\input{sections/G_experiments}

\FloatBarrier
\bibliography{refs}
\end{document}

%% file: sections/01_introduction.tex
\section{Introduction}
\label{sec:intro}
Online convex optimization models a learner that repeatedly chooses a point in a convex set and then suffers a convex loss at that point \citep{zinkevich2003online,hazan2016introduction}. When the environment changes over time, the learner is compared with a sequence of decisions rather than with a single decision, and the difficulty of this comparison is governed by how far the sequence moves. In many applications the learner also sees only a noisy gradient, and in neural-network training this noise has been observed to be heavy-tailed \citep{simsekli2019tail}. This paper determines the best achievable regret when both difficulties are present, and gives a single algorithm that, for each fixed moment order, attains it up to a constant factor without knowing the noise scale, the moment order, the size of the gradients, or how far the comparator moves.

We consider a compact convex set $\X\subset\R^d$ of diameter $D>0$ and convex losses $\ell_1,\ldots,\ell_T$ on $\X$ that are fixed in advance. Each loss admits a measurable subgradient selection $s_t$, that is, $s_t(x)\in\partial\ell_t(x)$ and $\norm{s_t(x)}\le G$ for every $x\in\X$. At round $t$ the learner plays $x_t\in\X$ and then observes one random vector $g_t$ whose conditional mean, given the past, is $s_t(x_t)$; it does not observe loss values. The noise $\epsilon_t=g_t-s_t(x_t)$ satisfies the unconditional moment bound $\E\norm{\epsilon_t}^p\le\sigma^p$ for some $p\in(1,2]$, so for $p<2$ its variance may be infinite, and its conditional moments given the past need not be bounded. For a comparator sequence $u_{1:T}\in\X^T$, the dynamic regret, the path length, and the normalized path length are
\begin{equation}
\label{eq:intro-path}
R_T(u_{1:T})=\sum_{t=1}^T\bigl(\ell_t(x_t)-\ell_t(u_t)\bigr),\qquad
P_T=\sum_{t=2}^T\norm{u_t-u_{t-1}},\qquad
\Lambda_T=1+\frac{P_T}{D}.
\end{equation}
Consecutive comparators are at most $D$ apart, so $1\le\Lambda_T\le T$, and $\Lambda_T=1$ exactly when the comparator is fixed.

Two special cases are well understood. With exact gradients ($\sigma=0$), the optimal dynamic regret is of order $GD\sqrt{T\Lambda_T}$, and the improved Ader algorithm of \citet{zhang2018adaptive} attains it without knowing $P_T$, using one gradient per round. \citet{gokcesu2024universal} replace $G\sqrt T$ in this bound by the square root of the sum of the squared norms of the observed gradients, again with one gradient per round and without a logarithmic factor of $T$. For a fixed comparator and heavy-tailed noise, \citet{liu2025heavy} shows that the optimal expected regret is of order $GD\sqrt T+\sigma DT^{1/p}$ and that AdaGrad \citep{duchi2011adaptive} attains it without knowing $G$, $\sigma$, or $p$. Noisy gradients and moving comparators have also been studied together. For noise with finite variance, \citet{yang2016tracking} obtain regret of order $\sqrt{TB_T}$ with gradient descent tuned to a known path-length budget $B_T$, and \citet{yang2024clipped} give high-probability bounds under heavy-tailed noise for clipped gradient methods whose schedules depend on $p$ and on a bound on the true gradients. These results do not characterize the minimax interaction between the noise moment and the path length, and none of them attains that interaction while adapting to all unknown parameters. In particular, the bound of \citet{gokcesu2024universal}, applied to the observed stochastic gradients, gives the noise term $\sigma DT^{1/p}\Lambda_T^{1/2}$, and we show that the exponent $1/2$ of $\Lambda_T$ in this term is not optimal when $p<2$.

\paragraph{Contributions.}
The paper makes three contributions.
\begin{itemize}[leftmargin=*]
\item \emph{The minimax rate.} Theorem~\ref{thm:minimax} shows that, on every compact convex domain of diameter $D>0$ and up to numerical constants that do not depend on $p$, the minimax expected regret over comparators whose normalized path length is at most $\Lambda_T$ is
\begin{equation}
\label{eq:intro-minimax}
\min\Bigl\{GD\sqrt{T\Lambda_T}+\sigma DT^{1/p}\Lambda_T^{(p-1)/p},\ GDT\Bigr\}.
\end{equation}
The lower bound holds even for algorithms that are told $G$, $\sigma$, $p$, and the path-length budget, and the upper bound is attained by a single restarted AdaGrad learner that is told the budget but not $G$, $\sigma$, or $p$.
\item \emph{Adaptation to the unknown parameters.} Algorithm~\ref{alg:main} is given only $\X$, $D$, and $T$, and it queries the oracle once per round. By Theorem~\ref{thm:main}(a), for every comparator sequence fixed independently of the learner's randomness and of the noise,
\begin{equation}
\label{eq:intro-rate}
\E R_T(u_{1:T})\le\min\Bigl\{C_0GD\sqrt{T\Lambda_T}+C_p\,\sigma DT^{1/p}\Lambda_T^{(p-1)/p},\ GDT\Bigr\},
\end{equation}
where $C_0$ is a numerical constant and $C_p=O(1+\log(p/(p-1)))$. For each fixed $p>1$ the bound contains no logarithmic factor of $T$ or $\Lambda_T$, so Algorithm~\ref{alg:main} attains the rate \eqref{eq:intro-minimax}. Theorem~\ref{thm:main}(b) gives a second bound whose constants are numerical uniformly over $p\in(1,2]$: the dynamic rate $3\bigl(GD\sqrt{T\Lambda_T}+\sigma DT^{1/p}\Lambda_T^{(p-1)/p}\bigr)$ is supplemented by the adaptation overhead $h(\Lambda_T)D(G\sqrt T+\sigma T^{1/p})$, with $h(\Lambda_T)=O(1+\log(1+\log\Lambda_T))$. By Corollary~\ref{cor:gap}, the expected regret is therefore at most $\min\{3+h(\Lambda_T),C_p\}$ times \eqref{eq:intro-minimax}, and by Theorem~\ref{thm:minimax} at most $20$ times this factor times the minimax value. Whether a factor that grows with $\Lambda_T$, or as $p\to1$, is necessary when the path length is unknown remains open.
\item \emph{Unknown horizon and confidence bounds.} Theorem~\ref{thm:anytime} removes $T$ from the inputs by restarting the algorithm on epochs of doubling length, which multiplies the bound by at most $2+\sqrt2$. Appendix~\ref{app:high-probability} gives a version of \eqref{eq:intro-rate} that holds with probability $1-\delta$, with the noise term multiplied by $O(\delta^{-1/p})$, and shows that the noise contribution $\sigma DT^{1/p}\Lambda_T^{(p-1)/p}\delta^{-1/p}$ cannot be improved under the unconditional moment bound.
\end{itemize}
Appendix~\ref{app:experiments} reports a synthetic study with Student-$t$ noise.

The two terms in \eqref{eq:intro-minimax} have different sources: the first is the price of following a moving comparator even with exact gradients, and the second is the additional price of noise. For $p=2$ both grow as $\Lambda_T^{1/2}$. For $p<2$ the noise term grows only as $\Lambda_T^{(p-1)/p}$. At $p=3/2$, for example, it equals $\sigma DT^{2/3}\Lambda_T^{1/3}$, whereas a bound over the whole horizon such as the one derived from \citet{gokcesu2024universal} gives $\sigma DT^{2/3}\Lambda_T^{1/2}$. In general, multiplying the static rate of \citet{liu2025heavy} by $\sqrt{\Lambda_T}$ overstates the cost of noise by the factor $\Lambda_T^{1/p-1/2}$. The minimum with $GDT$ is attained when the noise or the comparator movement is so large that no algorithm has sublinear regret.

\paragraph{Technical novelty.}
The central step is a blockwise calculation of the noise. A regret bound over the whole horizon of the form $D\sqrt{\Lambda_T}\,(\sum_t\norm{g_t}^2)^{1/2}$ has expectation at most $GD\sqrt{T\Lambda_T}+\sigma DT^{1/p}\Lambda_T^{1/2}$. AdaGrad restarted every $L$ rounds instead pays, on each block, the diameter plus the comparator movement inside the block, multiplied by the square root of the sum of squared gradient norms in that block, and this square root has expectation at most $G\sqrt L+\sigma L^{1/p}$. The noise part of the total cost is therefore of order $\sigma D(TL^{-(p-1)/p}+\Lambda_TL^{1/p})$: the first term is the cost of the $T/L$ restarts, and the second is the cost of movement inside the blocks. Taking $L$ of order $T/\Lambda_T$ balances the two and gives $\sigma DT^{1/p}\Lambda_T^{(p-1)/p}$, and the same $L$ gives $GD\sqrt{T\Lambda_T}$. This block length does not depend on $p$, so one set of block lengths serves every $p\in(1,2]$. Heavy-tailed noise is dominated by rare large errors whose contribution over $L$ rounds grows like $L^{1/p}$, faster than $\sqrt L$, and charging them block by block is what lowers the exponent of $\Lambda_T$.

The block length $T/\Lambda_T$ depends on the unknown path length. Algorithm~\ref{alg:main} therefore runs one restarted learner for every block length $2^i<2T$ and combines them by exponential weights, with a prior whose weight decreases polynomially with the distance from the longest block length. The learner suited to $P_T$ lies $O(1+\log\Lambda_T)$ positions from the longest block length, so its prior cost is $O(1+\log(1+\log\Lambda_T))$ and does not grow with $T$. Adapting to the path length without a logarithmic factor of $T$ is also achieved by \citet{gokcesu2024universal}. What is new here is that the prior cost multiplies the static rate $D(G\sqrt T+\sigma T^{1/p})$, whose two terms must be absorbed into $GD\sqrt{T\Lambda_T}$ and $\sigma DT^{1/p}\Lambda_T^{(p-1)/p}$, which have different exponents of $\Lambda_T$; this absorption produces the constant $C_p$ in \eqref{eq:intro-rate}.

The lower bound splits the horizon into about $\Lambda_T$ blocks of length $L$ of order $T/\Lambda_T$ and attaches an independent random sign to each block, which determines the direction of the loss on that block. The oracle reveals the sign of the current block only with probability $1/(2L)$ per round, which keeps the $p$-th moment of the noise below $\sigma^p$, and with probability at least $1/2$ the learner receives no information about the sign during the whole block. Each block then contributes regret $\Omega(D\min\{\sigma L^{1/p},GL\})$, and the blocks together contribute $\Omega(\min\{\sigma DT^{1/p}\Lambda_T^{(p-1)/p},GDT\})$.

The remaining ingredients are established, and we cite them beside the corresponding statements: the analysis of projected gradient steps against a moving comparator \citep{zinkevich2003online,hall2015online}, AdaGrad stepsizes \citep{duchi2011adaptive}, the AdaHedge combination rule \citep{derooij2014follow}, the reduction of Ader that updates all learners with one gradient observed at the combined point \citep{zhang2018adaptive}, and the conversion of a square-root gradient energy into a moment bound \citep{liu2025heavy}.

\begin{table}[t]
\caption{Comparison with the closest results. Every algorithm is given $\X$ and $D$, and all except AdaGrad and that of \citet{gokcesu2024universal} are also given $T$; the third column lists the further quantities each requires. The rates are $A=GD\sqrt{T\Lambda_T}$, $B=\sigma DT^{1/p}\Lambda_T^{(p-1)/p}$, $H=D(G\sqrt T+\sigma T^{1/p})$, and $C=GDT$. Stochastic bounds hold in expectation for each fixed comparator, and every upper bound may be replaced by its minimum with $C$. The third row is not stated by \citet{gokcesu2024universal}: it applies their pathwise bound to the linear losses $x\mapsto\inner{g_t,x}$ and then uses Proposition~\ref{prop:linearization} and Lemma~\ref{lem:moment}. The last row is a benchmark, not an algorithm.}
\label{tab:comparison}
\small
\setlength{\tabcolsep}{4pt}
\renewcommand{\arraystretch}{1.18}
\begin{tabular}{@{}>{\raggedright\arraybackslash}p{.24\textwidth}>{\raggedright\arraybackslash}p{.25\textwidth}>{\raggedright\arraybackslash}p{.18\textwidth}>{\raggedright\arraybackslash}p{\dimexpr.33\textwidth-6\tabcolsep\relax}@{}}
\toprule
Result & Feedback and comparator & Further inputs & Expected regret \\
\midrule
Improved Ader \citep{zhang2018adaptive} & Exact gradients; moving comparator & $G$ & $O(A)$ \\
AdaGrad \citep{liu2025heavy} & Stochastic gradients; fixed comparator & none & $O(H)$ \\
\citet{gokcesu2024universal}, transferred & Stochastic gradients; moving comparator & none & $O(A+\sigma DT^{1/p}\Lambda_T^{1/2})$ \\
Single restarted expert (Theorem~\ref{thm:minimax}) & Stochastic gradients; moving comparator & $P_T$ & $3(A+B)$ \\
\textbf{Algorithm~\ref{alg:main}} (Theorem~\ref{thm:main}) & Stochastic gradients; moving comparator & none & $C_0A+C_pB$ \\
\midrule
Minimax value (Theorem~\ref{thm:minimax}) & Stochastic gradients; path length at most $P_T$ & $G,\sigma,p,P_T$ may be used & $\Theta(\min\{A+B,C\})$ \\
\bottomrule
\end{tabular}
\end{table}

%% file: sections/02_related_work.tex
\section{Related Work}
\label{sec:related}

\paragraph{Dynamic regret with exact gradients.}
\citet{zinkevich2003online} introduced regret against a moving comparator and showed that online gradient descent with a stepsize tuned to the horizon has dynamic regret that grows linearly in $\Lambda_T$, and \citet{hall2015online} extended this analysis to mirror descent with dynamical models. \citet{zhang2018adaptive} proved that $\sqrt{T\Lambda_T}$ is the optimal order and attained it with Ader, which runs gradient descent with geometrically spaced stepsizes and combines the resulting learners by exponential weights. Their improved version queries one gradient at the combined point and gives every learner the same linear surrogate loss. Algorithm~\ref{alg:main} uses this reduction, but its learners differ in restart length rather than in stepsize, and each learner sets its stepsize from the gradients it has observed. \citet{gokcesu2024universal} combine adaptive learning rates with multi-scale aggregation. Their Theorem~5 bounds the dynamic regret on linear losses by a constant times $D\sqrt{\Lambda_T}\,(\sum_{t=1}^T\norm{g_t}^2)^{1/2}$ for every comparator sequence, with one gradient per round shared by all learners and without a logarithmic factor of $T$. Applied to the observed stochastic gradients, and combined with Proposition~\ref{prop:linearization} and Lemma~\ref{lem:moment} over the whole horizon, this bound gives the third row of Table~\ref{tab:comparison}; this transfer is ours. With heavy-tailed noise the sum of squared norms over the whole horizon can have infinite expectation, and bounding its square root over the whole horizon gives the noise term $\sigma DT^{1/p}\Lambda_T^{1/2}$ rather than $\sigma DT^{1/p}\Lambda_T^{(p-1)/p}$.

\paragraph{Dynamic regret with noisy gradients.}
\citet{yang2016tracking} study dynamic regret against the sequence of minimizers with noisy gradients of finite variance. With the path-length budget $B_T$ and the noise parameters known, gradient descent attains regret of order $\sqrt{TB_T}$, and they prove lower bounds of order $B_T^{\kappa}T^{1-\kappa}$ for $\kappa\in(1/2,1)$. Their lower-bound model assumes independent and identically distributed additive noise with bounded covariance, together with a regularity condition on the Kullback--Leibler divergence between shifted noise distributions. At $p=2$ our upper bound has the same square-root dependence without requiring $B_T$, $G$, or $\sigma$. Theorem~\ref{thm:lower-noise}, whose comparator is also a sequence of per-round minimizers, establishes the corresponding exponents for the broader class of conditionally unbiased oracles with bounded second moment considered here; its oracle need not satisfy their regularity condition. \citet{yang2024clipped} analyze distributed clipped stochastic gradient descent for smooth losses under heavy-tailed noise with bounded conditional $p$-th moments. They give high-probability bounds on the dynamic regret against the sequence of minimizers, of order $T^{(p+1)/(2p)}(1+C_T+\log(1/\delta))$ with $C_T$ the path length of the minimizers and other problem constants suppressed (Corollary~1 of arXiv:2401.14776), and their prescribed clipping and stepsize schedules depend on the moment order and on a bound on the true gradients. \citet{supantha2026universal} combine full-horizon AdaGrad experts by AdaHedge with a uniform prior, with exact gradients evaluated at every expert's point. If such experts were run on the observed stochastic gradients, a full-horizon energy bound of this kind would give a noise term of order $\sigma DT^{1/p}\Lambda_T^{1/2}$ with an additional $\log\log T$ factor from the uniform prior; this transfer is ours, and restarts together with a prior placed at the longest block length remove both factors.

\paragraph{Other measures of non-stationarity.}
Path length measures how far a comparator sequence moves. \citet{besbes2015non} instead bound the total variation of the loss functions and study non-stationary stochastic optimization under such a budget, which is a different class of problems. Interval regret and strongly adaptive regret require small regret against a fixed comparator on every subinterval of the horizon \citep{daniely2015strongly,cutkosky2020parameter}. Strongly adaptive methods run learners started at different times, which is also a form of restarting. For a comparator that is constant on each of $K$ deterministic intervals, interval bounds of order $D(\sum_{t\in I}\norm{g_t}^2)^{1/2}$ combined with Lemma~\ref{lem:moment} would give $O(GD\sqrt{KT}+\sigma DT^{1/p}K^{(p-1)/p})$. A comparator with small path length can, however, change at every round: on $\X=[0,D]$ the sequence $u_t=D(t-1)/(T-1)$ has $\Lambda_T=2$. For such comparators an additional argument must control the movement within each interval. Lemma~\ref{lem:block} supplies this control directly, so our bounds depend on the path length rather than on the number of switches. The analysis of \citet{cutkosky2020parameter} also assumes that the observed gradients are bounded, which heavy-tailed noise violates. Our guarantees concern the dynamic regret over the whole horizon against each fixed comparator, and the lower bound in Section~\ref{sec:lower} is proved for the same class of comparators.

\paragraph{Dynamic regret under additional structure.}
For smooth losses with exact gradients, \citet{zhao2020dynamic,zhao2024adaptivity} replace the factor $T$ in $\sqrt{T\Lambda_T}$ by the variation of the gradients or by the cumulative loss of the comparator, either of which can be much smaller than $T$. These problem-dependent bounds rely on smoothness and on exact gradients, whereas we treat general convex losses with noisy gradients, and Table~\ref{tab:comparison} therefore compares with the general convex rate.

\paragraph{Heavy-tailed stochastic optimization.}
Tail-index measurements of stochastic gradient noise in neural-network training \citep{simsekli2019tail} motivate noise models without a finite variance. Robust methods truncate gradients in mirror descent \citep{nazin2019robust} or clip them in accelerated methods \citep{gorbunov2020clipping}, and these works give high-probability guarantees when the second moment of the noise is bounded. \citet{vural2022mirror} obtain optimal rates for stochastic convex optimization when only a $p$-th moment with $p<2$ is finite, together with lower bounds for a single fixed objective. The result closest to ours is that of \citet{liu2025heavy}: on bounded domains, classical algorithms such as AdaGrad attain the optimal static expected regret under a finite $p$-th central moment without any clipping. We use the same conversion from a square-root energy bound to an expected bound, applied separately on each block. For temporally dependent samples, \citet{sarkar2026markovian} study clipped block methods under Markovian noise with finite $p$-th moments and give high-probability guarantees for smooth objectives that satisfy the Polyak--\L{}ojasiewicz condition. That work concerns the optimization error for one objective under mixing assumptions, whereas we study comparator movement for general convex losses with conditionally unbiased gradients, a property that Markovian noise need not have.

\paragraph{Parameter-free online learning.}
In much of the literature, a parameter-free algorithm is one that adapts to the unknown norm of the comparator on an unbounded domain \citep{jacobsen2022parameter}, and \citet{zhang2022parameter} obtain such guarantees with high probability under heavy-tailed gradients. In our setting the domain and its diameter are known, and the unknown quantities are $G$, $\sigma$, $p$, and $P_T$ (Definition~\ref{def:parameter-free}).

\paragraph{Adaptive learners, aggregation, and computation.}
AdaGrad \citep{duchi2011adaptive} scales projected gradient steps by the accumulated squared gradient norms. AdaHedge \citep{derooij2014follow} sets the learning rate of exponential weights from the accumulated mixability gaps, which makes it invariant to translations and positive rescalings of the losses. It coincides with AdaFTRL with the entropic regularizer \citep{orabona2018scale,hadiji2023adaptation}. The combination step of Algorithm~\ref{alg:main} is AdaHedge with a nonuniform prior, and Lemma~\ref{lem:master} restates its regret bound for arbitrary real loss vectors in terms of their realized ranges. The polynomial prior is taken from \citet{zhang2018adaptive}, with the order of the indices reversed so that longer block lengths receive more weight. Two-layer methods such as Ader and Algorithm~\ref{alg:main} project every learner onto the domain at every round; \citet{zhao2025efficient} reduce this to one projection onto the domain per round, and Remark~\ref{rem:one-projection} shows that their reduction applies to Algorithm~\ref{alg:main}.

%% file: sections/03_setup.tex
\section{Problem Setup}
\label{sec:setup}
Let $\X\subset\R^d$ be a nonempty closed convex set with Euclidean diameter
\begin{equation}
\label{eq:diameter}
D=\sup_{x,y\in\X}\norm{x-y}\in(0,\infty).
\end{equation}
As a closed and bounded subset of $\R^d$, the set $\X$ is compact, and we assume that the Euclidean projection $\Pi_\X$ onto $\X$ can be computed. The horizon is $T\ge2$. The losses $\ell_1,\ldots,\ell_T:\X\to\R$ are convex and fixed before the interaction. Each $\ell_t$ comes with a measurable map $s_t$ such that $s_t(x)\in\partial\ell_t(x)$ and $\norm{s_t(x)}\le G$ for all $x\in\X$, where $G\ge0$, and consequently every $\ell_t$ is $G$-Lipschitz on $\X$. Because the bound is imposed only on the selected subgradient $s_t(x)$, nonsmooth losses whose subdifferentials are unbounded at boundary points are allowed.

The learner's internal randomness is a random seed, and $\F_0$ denotes the $\sigma$-field it generates. After playing $x_t$ the learner observes a single vector $g_t\in\R^d$, so its information after round $t$ is $\F_t=\sigma(\F_0,g_1,\ldots,g_t)$, and the decision $x_t$ must be $\F_{t-1}$-measurable. The observation satisfies
\begin{equation}
\label{eq:oracle}
\E[g_t\mid\F_{t-1}]=s_t(x_t),
\end{equation}
and we write $\epsilon_t=g_t-s_t(x_t)$ for the noise. No other feedback is available; in particular, the learner does not observe loss values. For an unknown $p\in(1,2]$ and an unknown $\sigma\ge0$, the noise satisfies
\begin{equation}
\label{eq:moment-assumption}
\E\norm{\epsilon_t}^p\le\sigma^p,\qquad t=1,\ldots,T.
\end{equation}
Unbiasedness in \eqref{eq:oracle} is required conditionally on the past, whereas \eqref{eq:moment-assumption} bounds only the unconditional moment. The noise may therefore depend on the past, and a uniform bound on the conditional moments $\E[\norm{\epsilon_t}^p\mid\F_{t-1}]$ is a special case. Exact gradients correspond to $\sigma=0$, and for $p<2$ the noise may have infinite variance even though its mean is finite. Keeping $G$ and $\sigma$ separate lets the bounds distinguish the size of the true subgradients from the size of the noise.

A comparator sequence $u_{1:T}\in\X^T$ is fixed before the interaction, independently of the seed and of the noise, and it may depend on the losses; for example, it may minimize the total loss subject to a bound on the path length. Regret, path length, and $\Lambda_T$ are defined in \eqref{eq:intro-path}. Every deterministic inequality for a realized sequence of observed gradients holds simultaneously for all comparator sequences. The expectation and high-probability guarantees are stated separately for each comparator sequence fixed independently of the learner's seed and of the oracle noise.

The following reduction to linear losses is standard \citep{hazan2016introduction,liu2025heavy}: the first inequality follows from the subgradient inequality at $x_t$ and from \eqref{eq:oracle}, using that $x_t-u_t$ is $\F_{t-1}$-measurable and that $\E\norm{\epsilon_t}\le\sigma$ by Jensen's inequality, and the second follows from $\norm{s_t(x_t)}\le G$ and $\norm{x_t-u_t}\le D$.

\begin{proposition}[Linearization and the regret ceiling]
\label{prop:linearization}
Let $x_1,\ldots,x_T$ be the decisions of a learner such that each $x_t$ is $\F_{t-1}$-measurable, and let $u_{1:T}$ be a fixed comparator sequence. Under \eqref{eq:oracle} and \eqref{eq:moment-assumption} the first inequality below holds, and the second holds for every realization without these assumptions:
\begin{equation}
\label{eq:linearization}
\E R_T(u_{1:T})\le\E\sum_{t=1}^T\inner{g_t,x_t-u_t}
\qquad\text{and}\qquad
R_T(u_{1:T})\le GDT\ \text{ almost surely}.
\end{equation}
\end{proposition}

\begin{definition}[Parameter-free algorithm]
\label{def:parameter-free}
An algorithm is parameter-free if its only inputs are $\X$, $D$, $T$, and an arbitrary point $x_0\in\X$, and all its other computations use only the observed vectors $g_1,g_2,\ldots$. In particular, it uses none of $G$, $\sigma$, $p$, and the path length of the comparator.
\end{definition}

In Algorithm~\ref{alg:main} the horizon determines the set of block lengths, and Theorem~\ref{thm:anytime} removes it from the inputs at the cost of a constant factor. The diameter sets the stepsize of each learner.

For a path-length budget $0\le P\le D(T-1)$, let
\[
\mathcal U(P)=\Bigl\{u_{1:T}\in\X^T:\ \sum_{t=2}^T\norm{u_t-u_{t-1}}\le P\Bigr\};
\]
a budget larger than $D(T-1)$ gives the same set as $P=D(T-1)$. Let $\mathcal C_\X(G,\sigma,p)$ be the class of pairs $(\ell_{1:T},\mathcal O)$ in which $\ell_{1:T}$ is a loss sequence as above and $\mathcal O$ is an oracle that satisfies \eqref{eq:oracle} and \eqref{eq:moment-assumption} for every learner, that is, along every sequence of queries in which $x_t$ is $\F_{t-1}$-measurable. The minimax regret is
\begin{equation}
\label{eq:minimax-definition}
\Rate(\X;G,\sigma,p,P)
=\inf_{\mathcal A}\ \sup_{(\ell_{1:T},\mathcal O)\in\mathcal C_\X(G,\sigma,p)}\ \sup_{u_{1:T}\in\mathcal U(P)}\ \E_{\mathcal A,\mathcal O}\bigl[R_T(u_{1:T})\bigr],
\end{equation}
where the infimum is over all possibly randomized learners $\mathcal A$ in the model above, and $\mathcal A$ may depend on $G$, $\sigma$, $p$, and $P$. Because the losses are fixed, every comparator in the inner supremum is a deterministic sequence, and the expectation averages only over the oracle noise and the learner's seed.

For $\Lambda\in[1,T]$ we abbreviate the exact-gradient dynamic rate, the noise term of the dynamic rate, the static rate, and the regret ceiling by
\begin{equation}
\label{eq:rates}
\begin{aligned}
A_T(\Lambda)&=GD\sqrt{T\Lambda}, &\qquad B_T(\Lambda)&=\sigma DT^{1/p}\Lambda^{(p-1)/p},\\
H_T&=D\bigl(G\sqrt T+\sigma T^{1/p}\bigr), &\qquad C_T&=GDT.
\end{aligned}
\end{equation}
Logarithms are natural unless written as $\log_2$.

%% file: sections/04_algorithm.tex
\section{Algorithm}
\label{sec:algorithm}
Algorithm~\ref{alg:main} runs $M=N+1$ learners, called experts, and combines their points by exponential weights. Expert $i$ runs AdaGrad and restarts it every $L_i$ rounds, where
\begin{equation}
\label{eq:pool}
N=\lceil\log_2T\rceil,\qquad L_i=2^i,\qquad i=0,1,\ldots,N.
\end{equation}
Expert $i$ restarts at rounds $1,1+L_i,1+2L_i,\ldots$; at each restart it moves to a fixed point $x_0\in\X$ and resets its accumulated energy to zero. If its current block began at round $s$, then at every round $t$ of that block it sets
\begin{equation}
\label{eq:expert}
V_t^{(i)}=\sum_{r=s}^t\norm{g_r}^2,\qquad
x_{t+1}^{(i)}=\Pi_\X\!\left(x_t^{(i)}-\frac{D\,g_t}{\sqrt{2V_t^{(i)}}}\right),
\end{equation}
except that it keeps $x_{t+1}^{(i)}=x_t^{(i)}$ when $V_t^{(i)}=0$, which happens only if every vector observed in the block so far is zero. At the last round of a block, and at round $T$, the point $x_{t+1}^{(i)}$ is discarded (at a block end it is replaced by $x_0$, so each block starts afresh), and an implementation of these cold restarts may skip that projection.

The master combines the experts using the prior
\begin{equation}
\label{eq:prior}
\pi_i=\frac{N+2}{N+1}\cdot\frac{1}{(N-i+1)(N-i+2)},\qquad i=0,\ldots,N.
\end{equation}
Summing the identity $1/((j+1)(j+2))=1/(j+1)-1/(j+2)$ over $j=N-i\in\{0,\ldots,N\}$ gives $(N+1)/(N+2)$, so $\pi$ is a probability vector. The longest block length receives weight $\pi_N=(N+2)/(2N+2)\ge1/2$, and the weight of expert $N-j$ decreases as $(j+1)^{-2}$. This is the polynomial prior of \citet{zhang2018adaptive}, indexed here from the longest block length.

Let $S_{t-1}\in\R^M$ be the vector of cumulative expert losses kept by the master and let $\lambda_t\ge0$ be its regularization parameter, both defined in \eqref{eq:master-update} below. The weights, the played point, and the expert losses at round $t$ are
\begin{align}
\label{eq:weights}
w_{t,i}&=
\begin{cases}
\pi_i, &\lambda_t=0,\\[2pt]
\dfrac{\pi_i\exp(-S_{t-1,i}/\lambda_t)}{\sum_j\pi_j\exp(-S_{t-1,j}/\lambda_t)}, &\lambda_t>0,
\end{cases}\\
\label{eq:mixture}
x_t&=\sum_{i=0}^Nw_{t,i}\,x_t^{(i)},\qquad
m_{t,i}=\inner{g_t,x_t^{(i)}-x_t}.
\end{align}
Because $x_t$ is the $w_t$-average of the expert points, $\inner{w_t,m_t}=0$. The loss $m_{t,i}$ is the linear loss $\inner{g_t,x_t^{(i)}}$ shifted by the same amount $\inner{g_t,x_t}$ for every expert, and such a shift does not change the weights. After observing $g_t$, the master computes the mixability gap
\begin{equation}
\label{eq:gap}
\delta_t=
\begin{cases}
\inner{\pi,m_t}-\min_im_{t,i}, &\lambda_t=0,\\[2pt]
\inner{w_t,m_t}+\lambda_t\log\Bigl(\sum_iw_{t,i}\exp(-m_{t,i}/\lambda_t)\Bigr), &\lambda_t>0,
\end{cases}
\end{equation}
and updates
\begin{equation}
\label{eq:master-update}
S_t=S_{t-1}+m_t,\qquad \lambda_{t+1}=\lambda_t+\delta_t,\qquad\text{with } S_0=0\text{ and }\lambda_1=0.
\end{equation}
In Algorithm~\ref{alg:main}, $\inner{w_t,m_t}=0$ at every round, and $w_t=\pi$ when $\lambda_t=0$, so the inner product in either case of \eqref{eq:gap} vanishes; we keep the general form because Lemma~\ref{lem:master} is stated for arbitrary vectors. This is the AdaHedge rule of \citet{derooij2014follow} with prior $\pi$: the learning rate $1/\lambda_t$ is determined by the gaps observed so far, and no bound on the size of the losses is needed. Lemma~\ref{lem:master} shows that $\delta_t\ge0$, so $\lambda_t$ is nondecreasing.

\begin{algorithm2e}[t]
\caption{Restarted AdaGrad experts with an adaptive master}
\label{alg:main}
\DontPrintSemicolon
\KwIn{$\X$, its diameter $D>0$, horizon $T\ge2$, and $x_0\in\X$}
Set $N=\lceil\log_2T\rceil$, $L_i=2^i$, and the prior $\pi$ in \eqref{eq:prior}\;
Set $S_0=0$ and $\lambda_1=0$\;
\For{$t=1,\ldots,T$}{
  \For{$i=0,\ldots,N$}{
    \If{$t-1$ is divisible by $L_i$}{
      Set $x_t^{(i)}=x_0$ and $V^{(i)}=0$\;
    }
  }
  Compute $w_t$ from \eqref{eq:weights} and play $x_t=\sum_iw_{t,i}x_t^{(i)}$\;
  Observe one stochastic subgradient $g_t$ at $x_t$\;
  Set $m_{t,i}=\inner{g_t,x_t^{(i)}-x_t}$ for each $i$\;
  Compute $\delta_t$ from \eqref{eq:gap}, and set $S_t=S_{t-1}+m_t$ and $\lambda_{t+1}=\lambda_t+\delta_t$\;
  \For{$i=0,\ldots,N$}{
    Set $V^{(i)}\gets V^{(i)}+\norm{g_t}^2$\;
    \eIf{$V^{(i)}>0$}{
      Set $x_{t+1}^{(i)}=\Pi_\X(x_t^{(i)}-Dg_t/\sqrt{2V^{(i)}})$\;
    }{
      Set $x_{t+1}^{(i)}=x_t^{(i)}$\;
    }
  }
}
\end{algorithm2e}

The oracle is queried once per round, at the played point $x_t$, and every expert is updated with the same vector $g_t$. The experts and the master therefore all work with the linear loss $x\mapsto\inner{g_t,x}$. Inequalities for this linear loss hold for every realization of $g_t$, even though the conditional mean of $g_t$ is a subgradient only at $x_t$, and Proposition~\ref{prop:linearization} converts their sum into a bound on the expected regret for the original losses. This reduction is the one used by the improved version of Ader \citep{zhang2018adaptive}.

The restart times do not depend on the observations. By induction on $t$, the expert points, the weights, and $x_t$ are functions of the seed and of $g_1,\ldots,g_{t-1}$; since projection onto $\X$ is continuous and the remaining operations are measurable, $x_t$ is $\F_{t-1}$-measurable, and Proposition~\ref{prop:linearization} applies to Algorithm~\ref{alg:main}. Each round uses one oracle query, at most $M=\lceil\log_2T\rceil+1$ projections, and $O(dM)$ further arithmetic operations and memory. The weights and the gap can be evaluated stably in log-sum-exp form, and the algorithm needs neither a clipping threshold nor an estimate of the noise moments.

Algorithm~\ref{alg:main} is also invariant to positive rescaling of the observations. If every $g_t$ is replaced by $cg_t$ with $c>0$, each expert energy is multiplied by $c^2$, so the normalized step in \eqref{eq:expert} is unchanged. The losses $m_t$, the cumulative losses $S_t$, the gaps $\delta_t$, and the parameters $\lambda_t$ are all multiplied by $c$, so the weights \eqref{eq:weights}, and with them every decision, are unchanged.

The following three remarks record variants for which the guarantees of Section~\ref{sec:upper} continue to hold.

\begin{remark}[Warm restarts]
\label{rem:warm}
Lemma~\ref{lem:block} allows an arbitrary first point $y_s\in\X$ in each block. All results of Section~\ref{sec:upper} therefore remain true if, at a restart, expert $i$ keeps its current point and resets only its energy $V^{(i)}$; the projection at the end of each block is then needed. In the study of Appendix~\ref{app:experiments}, the reported mean regrets of the two variants differ by about $6\%$ or less.
\end{remark}

\begin{remark}[One projection per round]
\label{rem:one-projection}
The reduction of \citet{zhao2025efficient} applies to Algorithm~\ref{alg:main}. Let $\mathcal Y=\{y:\norm{y-x_0}\le D\}$, a ball of diameter $2D$ that contains $\X$. Run the experts and the master on $\mathcal Y$, with diameter $2D$ in place of $D$, so that their projections are onto a ball and have a closed form. Let $y_t$ be the combined point, play $x_t=\Pi_\X(y_t)$, observe $g_t$ at $x_t$, and give the experts and the master the vector $\tilde g_t=g_t-\min\{0,\inner{g_t,v_t}\}v_t$, where $v_t=(y_t-x_t)/\norm{y_t-x_t}$ if $y_t\ne x_t$, and $\tilde g_t=g_t$ otherwise. Appendix~\ref{app:one-projection} shows that $\norm{\tilde g_t}\le\norm{g_t}$ and that $\sum_t\inner{g_t,x_t-u_t}\le\sum_t\inner{\tilde g_t,y_t-u_t}$ for every comparator sequence in $\X$. Consequently all upper bounds hold with $D$ replaced by $2D$ in the rates, which at most doubles them, while each round uses one projection onto $\X$.
\end{remark}

\begin{remark}[Unknown diameter]
\label{rem:diameter}
If only an upper bound $\hat D\ge D$ is known, Algorithm~\ref{alg:main} can be run with $\hat D$ in place of $D$. The analysis uses $D$ only as an upper bound on distances between points of $\X$, so all upper bounds hold with $D$ replaced by $\hat D$, including in the normalized path length $1+P_T/\hat D$.
\end{remark}

%% file: sections/05_upper_bounds.tex
\section{Upper Bounds}
\label{sec:upper}
\begin{theorem}[Parameter-free dynamic regret]
\label{thm:main}
Assume \eqref{eq:oracle} and \eqref{eq:moment-assumption}, and let $u_{1:T}$ be a comparator sequence fixed before the interaction, independently of the learner's seed and of the oracle noise. Algorithm~\ref{alg:main} satisfies the following two bounds.
\begin{enumerate}[label=(\alph*),leftmargin=*]
\item For every $p\in(1,2]$,
\begin{equation}
\label{eq:logfree-main}
\E R_T(u_{1:T})\le\min\bigl\{C_0A_T(\Lambda_T)+C_pB_T(\Lambda_T),\ C_T\bigr\},
\end{equation}
where $C_0<9.3$ is a numerical constant and $C_p$ depends only on $p$ and satisfies $C_p=O(1+\log(p/(p-1)))$; both are given explicitly in \eqref{eq:constants}.
\item For every $p\in(1,2]$,
\begin{equation}
\label{eq:uniform-main}
\E R_T(u_{1:T})\le\min\Bigl\{3\bigl(A_T(\Lambda_T)+B_T(\Lambda_T)\bigr)+h(\Lambda_T)H_T,\ C_T\Bigr\},
\end{equation}
where, for $\Lambda\ge1$,
\begin{equation}
\label{eq:h}
h(\Lambda)=\frac{\sqrt5}{2}\Bigl[1+2\log\bigl(4+\log_2\Lambda\bigr)\Bigr].
\end{equation}
\end{enumerate}
\end{theorem}

The proof is given in Section~\ref{subsec:proof-main}. For each fixed $p>1$, part~(a) contains no logarithmic factor of $T$ or $\Lambda_T$. When $\sigma=0$ it gives the optimal exact-gradient rate $O(GD\sqrt{T\Lambda_T})$ with a numerical constant, and when $p=2$ it gives $O((G+\sigma)D\sqrt{T\Lambda_T})$. The constant $C_p$ grows like $\log(1/(p-1))$ as $p\to1$, because the proof of part~(a) bounds $h(\Lambda_T)$ by a multiple of $\Lambda_T^{(p-1)/p}$. Part~(b) does not use this step. In part~(b), the function $h$ bounds the cost paid by the master for the expert used in the analysis; the constants are numerical for all $p\in(1,2]$, and the factor $h(\Lambda_T)=O(1+\log(1+\log\Lambda_T))$ multiplies only the static rate $H_T$.

The two parts can be combined into a single comparison with the rate \eqref{eq:intro-minimax}.

\begin{corollary}[Distance to the benchmark rate]
\label{cor:gap}
Under the assumptions of Theorem~\ref{thm:main},
\begin{equation}
\label{eq:gap-to-minimax}
\E R_T(u_{1:T})\le\min\bigl\{3+h(\Lambda_T),\,C_p\bigr\}\cdot\min\bigl\{A_T(\Lambda_T)+B_T(\Lambda_T),\,C_T\bigr\}.
\end{equation}
\end{corollary}

The factor $\min\{3+h(\Lambda_T),C_p\}$ in \eqref{eq:gap-to-minimax} compares Algorithm~\ref{alg:main} with the benchmark rate $\min\{A_T(\Lambda_T)+B_T(\Lambda_T),C_T\}$. It is at most of order $1+\log(1+\log\Lambda_T)$ uniformly in $p$, and at most of order $1+\log(p/(p-1))$ uniformly in $\Lambda_T$ and $T$. Theorem~\ref{thm:minimax} shows that the minimax regret lies between $1/20$ and $3$ times the benchmark rate, so relative to the minimax value the factor is at most $20\min\{3+h(\Lambda_T),C_p\}$; this is the price that Algorithm~\ref{alg:main} pays in our analysis for not knowing $G$, $\sigma$, $p$, and $P_T$. Appendix~\ref{app:high-probability} shows that part~(a) also holds with probability at least $1-\delta$ after $C_p$ is replaced by $(C_p+2^{1/p})\delta^{-1/p}$, and that, for each fixed $p$, the joint dependence of the noise term on $\Lambda_T$ and $\delta$ cannot be improved under \eqref{eq:moment-assumption}.

For a fixed comparator, the expert with the longest block length gives a bound with a numerical constant.

\begin{corollary}[Static regret]
\label{cor:static}
For every fixed $u\in\X$, Algorithm~\ref{alg:main} satisfies
\begin{equation}
\label{eq:static}
\E R_T(u,\ldots,u)\le\min\Bigl\{\Bigl[\sqrt2+\frac{\sqrt5}{2}(1+\log2)\Bigr]H_T,\ C_T\Bigr\}.
\end{equation}
\end{corollary}

The proof is given at the end of Section~\ref{subsec:proof-main}.

Algorithm~\ref{alg:main} needs $T$ only to fix its block lengths. To remove this input, we run independent copies of Algorithm~\ref{alg:main} on consecutive epochs. Epoch $j=0,1,\ldots$ has planned length $T_j$ and first round $\tau_j$ given by
\begin{equation}
\label{eq:epoch-schedule}
T_j=2^{j+1},\qquad \tau_j=1+\sum_{k<j}T_k=2^{j+1}-1.
\end{equation}
At round $\tau_j$ the copy used in the previous epoch is discarded, and a new copy of Algorithm~\ref{alg:main} with horizon $T_j$ and initial point $x_0$ is started.

\begin{theorem}[Unknown horizon]
\label{thm:anytime}
Assume that \eqref{eq:oracle} and \eqref{eq:moment-assumption} hold at every round. The epoch procedure \eqref{eq:epoch-schedule} receives only $\X$, $D$, and $x_0$, and for every $T\ge2$ and every fixed comparator sequence $u_{1:T}$ it satisfies the following two bounds.
\begin{enumerate}[label=(\alph*),leftmargin=*]
\item $\E R_T(u_{1:T})\le\min\Bigl\{(2+\sqrt2)\bigl[C_0A_T(\Lambda_T)+C_pB_T(\Lambda_T)\bigr],\ C_T\Bigr\}$.
\item $\E R_T(u_{1:T})\le\min\Bigl\{(2+\sqrt2)\Bigl[3\bigl(A_T(\Lambda_T)+B_T(\Lambda_T)\bigr)+h(\Lambda_T)H_T\Bigr],\ C_T\Bigr\}$.
\end{enumerate}
At round $t$ the procedure stores one copy of Algorithm~\ref{alg:main} with $O(\log(t+1))$ experts, and it uses one oracle query and $O(\log(t+1))$ projections.
\end{theorem}

The proof is given in Appendix~\ref{app:anytime}. The rest of this section proves Theorem~\ref{thm:main} and Corollaries~\ref{cor:gap} and~\ref{cor:static}.

\subsection{One restarted expert}
\label{subsec:block}
For a block $B=\{s,\ldots,e\}$ of consecutive rounds we write
\begin{equation}
\label{eq:block-notation}
V_B=\sum_{t=s}^e\norm{g_t}^2,\qquad
V_{B,<t}=\sum_{r=s}^{t-1}\norm{g_r}^2,\qquad
P_B=\sum_{t=s+1}^e\norm{u_t-u_{t-1}}.
\end{equation}
The quantity $P_B$ counts only the movement inside the block; a change of comparator between two blocks is paid for by the diameter term of the later block in the next lemma.

\begin{lemma}[Regret of one restarted expert]
\label{lem:block}
Let $\mathcal K\subset\R^d$ be a closed convex set of diameter at most $D$, let $B=\{s,\ldots,e\}$, and let $g_s,\ldots,g_e\in\R^d$ be arbitrary vectors. Starting from an arbitrary $y_s\in\mathcal K$, let $y_{t+1}=\Pi_{\mathcal K}\bigl(y_t-Dg_t/\sqrt{2V_{B,<t+1}}\bigr)$ if $V_{B,<t+1}>0$ and $y_{t+1}=y_t$ otherwise, for $t=s,\ldots,e-1$. Then for all $u_s,\ldots,u_e\in\mathcal K$,
\begin{equation}
\label{eq:block-refined}
\sum_{t=s}^e\inner{g_t,y_t-u_t}
\le\sqrt2D\sqrt{V_B}+\sqrt2\sum_{t=s+1}^e\norm{u_t-u_{t-1}}\sqrt{V_{B,<t}}
\le\sqrt2(D+P_B)\sqrt{V_B}.
\end{equation}
In particular, \eqref{eq:block-refined} holds on every block $B$ of expert $i$ in Algorithm~\ref{alg:main}, with $\mathcal K=\X$, $y_s=x_0$, and $y_t=x_t^{(i)}$.
\end{lemma}

Lemma~\ref{lem:block} combines the analysis of projected gradient steps against a moving comparator \citep{zinkevich2003online,hall2015online} with the AdaGrad stepsize and its summation argument \citep{duchi2011adaptive}. In the middle expression of \eqref{eq:block-refined}, each movement of the comparator is multiplied only by the gradient energy observed before that movement. This first inequality records the timing of the movement charge, and it is used in Proposition~\ref{prop:pathwise}; the second is the form used to derive the rates. The proof, given in Appendix~\ref{app:aux}, follows these analyses; the adaptation is the regrouping that charges each movement to the energy accumulated before it.

\subsection{The master}
\label{subsec:master}
\begin{lemma}[Master regret with a prior]
\label{lem:master}
Let $m_1,\ldots,m_T\in\R^M$ be arbitrary vectors and let $\pi$ be a probability vector with positive entries. Run \eqref{eq:weights}, \eqref{eq:gap}, and \eqref{eq:master-update} on these vectors, and let
\begin{equation}
\label{eq:range}
r_t=\max_im_{t,i}-\min_im_{t,i},\qquad b_i=\frac{\sqrt5}{2}\Bigl(1+\log\frac1{\pi_i}\Bigr).
\end{equation}
Then $\delta_t\ge0$ for every $t$, and for every expert $i$,
\begin{equation}
\label{eq:master-bound}
\sum_{t=1}^T\bigl(\inner{w_t,m_t}-m_{t,i}\bigr)\le b_i\Bigl(\sum_{t=1}^Tr_t^2\Bigr)^{1/2}.
\end{equation}
\end{lemma}

Lemma~\ref{lem:master} is the regret bound of AdaHedge \citep{derooij2014follow}, equivalently AdaFTRL with the entropic regularizer \citep{orabona2018scale,hadiji2023adaptation}, written for a nonuniform prior and for loss vectors with arbitrary real entries. Appendix~\ref{app:master-proof} gives the proof, which follows \citet{derooij2014follow}; the adaptations are the prior, which enters through the relative entropy to $\pi$, and the convention for $\lambda_t=0$, which allows unbounded losses. The bound holds for vectors of any size, which matters because $g_t$, and hence $m_t$, may have infinite variance.

In Algorithm~\ref{alg:main}, $m_{t,i}-m_{t,j}=\inner{g_t,x_t^{(i)}-x_t^{(j)}}$, so the Cauchy--Schwarz inequality and the diameter bound give
\begin{equation}
\label{eq:range-geometric}
r_t=\max_{i,j}\inner{g_t,x_t^{(i)}-x_t^{(j)}}\le D\norm{g_t}.
\end{equation}

\subsection{Comparison with each expert}
\label{subsec:pathwise}
For an integer $L\ge1$, let $\mathcal B(L)$ be the partition of $\{1,\ldots,T\}$ into consecutive blocks of length $L$, the last of which may be shorter; the blocks of expert $i$ form $\mathcal B(L_i)$. Let $V_{[T]}=\sum_{t=1}^T\norm{g_t}^2$ and
\begin{equation}
\label{eq:expert-cost}
\mathcal E_i(u_{1:T})=\sum_{B\in\mathcal B(L_i)}\Biggl[D\sqrt{V_B}+\sum_{t\in B\setminus\{\min B\}}\norm{u_t-u_{t-1}}\sqrt{V_{B,<t}}\Biggr].
\end{equation}

\begin{proposition}[Comparison with each expert]
\label{prop:pathwise}
For every realization of $g_1,\ldots,g_T$, every comparator sequence $u_{1:T}\in\X^T$, and every $i\in\{0,\ldots,N\}$, Algorithm~\ref{alg:main} satisfies
\begin{equation}
\label{eq:pathwise}
\begin{aligned}
\sum_{t=1}^T\inner{g_t,x_t-u_t}
&\le\sqrt2\,\mathcal E_i(u_{1:T})+b_i\Bigl(\sum_{t=1}^Tr_t^2\Bigr)^{1/2}\\
&\le\sqrt2\sum_{B\in\mathcal B(L_i)}(D+P_B)\sqrt{V_B}+b_iD\sqrt{V_{[T]}}.
\end{aligned}
\end{equation}
Consequently, under \eqref{eq:oracle} and \eqref{eq:moment-assumption}, for every fixed comparator sequence and every $i$,
\begin{equation}
\label{eq:pathwise-expected}
\E R_T(u_{1:T})\le\sqrt2\sum_{B\in\mathcal B(L_i)}(D+P_B)\,\E\sqrt{V_B}+b_iD\,\E\sqrt{V_{[T]}}.
\end{equation}
\end{proposition}

Proposition~\ref{prop:pathwise} follows the two-step structure of the analysis of Ader \citep{zhang2018adaptive}: the regret of the combined point is split into the regret of the master against expert $i$ and the regret of expert $i$, both on the same linear losses. The proof is given in Appendix~\ref{app:aux}.

\subsection{Expected energy on blocks}
\label{subsec:energy}
The next lemma converts the square-root energy of a block into the moment scale. \citet{liu2025heavy} proves it for the whole horizon; we need it for an arbitrary deterministic set of rounds and under the unconditional bound \eqref{eq:moment-assumption}, and Appendix~\ref{app:aux} gives the proof, which follows \citet{liu2025heavy}.

\begin{lemma}[Blockwise gradient energy]
\label{lem:moment}
Under \eqref{eq:oracle} and \eqref{eq:moment-assumption}, every deterministic set $B\subseteq\{1,\ldots,T\}$ satisfies
\begin{equation}
\label{eq:moment-energy}
\E\Bigl(\sum_{t\in B}\norm{g_t}^2\Bigr)^{1/2}\le G\sqrt{|B|}+\sigma|B|^{1/p}.
\end{equation}
\end{lemma}

The set $B$ must not depend on the observations, and this is why the restart times of Algorithm~\ref{alg:main} are fixed in advance. The next lemma evaluates the first term of \eqref{eq:pathwise-expected} for a single block length. It is used both in the proof of Theorem~\ref{thm:main} and for the upper bound of Theorem~\ref{thm:minimax}, where only an upper bound on the path length is available.

\begin{lemma}[Expected cost of a restarted expert]
\label{lem:scale}
Assume \eqref{eq:oracle} and \eqref{eq:moment-assumption}, fix a comparator sequence, and let $\Lambda\in[1,T]$ satisfy $P_T\le D(\Lambda-1)$. For every integer $L\ge1$, with $\gamma_L=G\sqrt L+\sigma L^{1/p}$,
\begin{equation}
\label{eq:scale-general}
\sqrt2\sum_{B\in\mathcal B(L)}(D+P_B)\,\E\sqrt{V_B}\le\sqrt2D\gamma_L\Bigl(\frac TL+\Lambda\Bigr).
\end{equation}
If moreover $L=2^{\lfloor\log_2(T/\Lambda)\rfloor}$, then the right side of \eqref{eq:scale-general} is at most $3\bigl(A_T(\Lambda)+B_T(\Lambda)\bigr)$.
\end{lemma}

\begin{proof}
The partition $\mathcal B(L)$ has $K=\lceil T/L\rceil\le T/L+1$ blocks. Each block is a deterministic set of at most $L$ rounds, so Lemma~\ref{lem:moment} gives $\E\sqrt{V_B}\le G\sqrt{|B|}+\sigma|B|^{1/p}\le\gamma_L$. The numbers $P_B$ count movements inside disjoint blocks, so $\sum_{B\in\mathcal B(L)}P_B\le P_T\le D(\Lambda-1)$. Therefore
\[
\sum_{B\in\mathcal B(L)}(D+P_B)\,\E\sqrt{V_B}\le\gamma_L\Bigl(DK+\sum_{B\in\mathcal B(L)}P_B\Bigr)\le\gamma_L\Bigl(D\Bigl(\frac TL+1\Bigr)+D(\Lambda-1)\Bigr)=D\gamma_L\Bigl(\frac TL+\Lambda\Bigr),
\]
and multiplying by $\sqrt2$ gives \eqref{eq:scale-general}.

Now let $L=2^{\lfloor\log_2(T/\Lambda)\rfloor}$. Since $T/\Lambda\ge1$, $L$ is a positive integer with $T/(2\Lambda)<L\le T/\Lambda$, so $L=rT/\Lambda$ for some $r\in(1/2,1]$. Then $T/L+\Lambda=\Lambda(1+1/r)$, and the two parts of the right side of \eqref{eq:scale-general} become
\begin{align*}
\sqrt2DG\sqrt L\,\Lambda\Bigl(1+\frac1r\Bigr)&=\sqrt2\Bigl(\sqrt r+\frac1{\sqrt r}\Bigr)GD\sqrt{T\Lambda}=\sqrt2\Bigl(\sqrt r+\frac1{\sqrt r}\Bigr)A_T(\Lambda),\\
\sqrt2D\sigma L^{1/p}\,\Lambda\Bigl(1+\frac1r\Bigr)&=\sqrt2\,r^{1/p}\Bigl(1+\frac1r\Bigr)\sigma DT^{1/p}\Lambda^{1-1/p}=\sqrt2\,r^{1/p}\Bigl(1+\frac1r\Bigr)B_T(\Lambda).
\end{align*}
Because $r\le1$ and $1/p\ge1/2$, we have $r^{1/p}\le r^{1/2}$, so the second coefficient is at most the first. The function $\phi(r)=\sqrt r+1/\sqrt r$ satisfies $\phi'(r)=(r-1)/(2r^{3/2})\le0$ for $r\in[1/2,1]$, so it is nonincreasing on this interval and $\phi(r)\le\phi(1/2)=3/\sqrt2$. Both coefficients are therefore at most $\sqrt2\cdot3/\sqrt2=3$.
\end{proof}

\subsection{Proofs of Theorem~\ref{thm:main} and Corollaries~\ref{cor:gap} and~\ref{cor:static}}
\label{subsec:proof-main}
The proof of Theorem~\ref{thm:main}(a) uses the function, defined for $a>0$,
\begin{equation}
\label{eq:kappa}
\kappa(a)=\frac{\sqrt5}{2}\Bigl[1+2\log\Bigl(4+\frac{1}{a\log2}\Bigr)+\frac2e\Bigr],
\end{equation}
which is decreasing in $a$, and the constants
\begin{equation}
\label{eq:constants}
a_p=\frac{p-1}{p},\qquad C_0=3+\kappa(1/2),\qquad C_p=3+\kappa(a_p).
\end{equation}

\begin{proof}[Proof of Theorem~\ref{thm:main}]
Fix the comparator sequence and write $\Lambda=\Lambda_T\in[1,T]$. We prove part~(b) first, because part~(a) is derived from it.

\textit{Proof of part (b).} Let $i^\star=\lfloor\log_2(T/\Lambda)\rfloor$. Since $1\le T/\Lambda\le T$ and $N=\lceil\log_2T\rceil$, we have $0\le i^\star\le N$, so $L_{i^\star}=2^{i^\star}$ is one of the block lengths of Algorithm~\ref{alg:main}. Inequality \eqref{eq:pathwise-expected} with $i=i^\star$ gives
\begin{equation}
\label{eq:main-split}
\E R_T(u_{1:T})\le\sqrt2\sum_{B\in\mathcal B(L_{i^\star})}(D+P_B)\,\E\sqrt{V_B}+b_{i^\star}D\,\E\sqrt{V_{[T]}}.
\end{equation}
Since $P_T=D(\Lambda-1)$, Lemma~\ref{lem:scale} applies with this $\Lambda$ and $L=L_{i^\star}$, and it bounds the first term on the right of \eqref{eq:main-split} by $3(A_T(\Lambda)+B_T(\Lambda))$.

For the second term we bound the prior cost of expert $i^\star$. Let $j=N-i^\star\ge0$. From $N<\log_2T+1$ and $i^\star>\log_2(T/\Lambda)-1$ we get $j<2+\log_2\Lambda$. By \eqref{eq:prior},
\begin{equation}
\label{eq:prior-cost}
\log\frac1{\pi_{i^\star}}=\log\frac{N+1}{N+2}+\log(j+1)+\log(j+2)\le2\log(j+2)<2\log(4+\log_2\Lambda),
\end{equation}
where the first inequality uses $(N+1)/(N+2)<1$ and $j+1<j+2$. Comparing \eqref{eq:range} with \eqref{eq:h}, this gives $b_{i^\star}\le h(\Lambda)$. Lemma~\ref{lem:moment} with $B=\{1,\ldots,T\}$ gives $D\,\E\sqrt{V_{[T]}}\le D(G\sqrt T+\sigma T^{1/p})=H_T$, so the second term on the right of \eqref{eq:main-split} is at most $h(\Lambda)H_T$. Hence $\E R_T(u_{1:T})\le3(A_T(\Lambda)+B_T(\Lambda))+h(\Lambda)H_T$. By the second part of Proposition~\ref{prop:linearization}, $R_T(u_{1:T})\le C_T$ almost surely, so $\E R_T(u_{1:T})\le C_T$, and \eqref{eq:uniform-main} follows.

\textit{Proof of part (a).} We first show that for every $a>0$,
\begin{equation}
\label{eq:absorb}
h(\Lambda)\le\kappa(a)\Lambda^a\qquad\text{for all }\Lambda\ge1.
\end{equation}
Let $y=\log\Lambda\ge0$, so that $\log_2\Lambda=y/\log2$ and $\Lambda^a=e^{ay}$, and let $c_a=4+1/(a\log2)$. Expanding the product,
\[
c_a(1+ay)=4+4ay+\frac{1}{a\log2}+\frac{y}{\log2}\ge4+\frac{y}{\log2},
\]
so $\log(4+\log_2\Lambda)\le\log c_a+\log(1+ay)\le\log c_a+ay$, where the last step uses $\log(1+z)\le z$. Substituting into \eqref{eq:h},
\[
h(\Lambda)\le\frac{\sqrt5}{2}\bigl[1+2\log c_a+2ay\bigr].
\]
Multiplying by $\Lambda^{-a}=e^{-ay}\le1$ and using $ze^{-z}\le1/e$ for $z=ay\ge0$,
\[
\Lambda^{-a}h(\Lambda)\le\frac{\sqrt5}{2}\Bigl[1+2\log c_a+2aye^{-ay}\Bigr]\le\frac{\sqrt5}{2}\Bigl[1+2\log c_a+\frac2e\Bigr]=\kappa(a),
\]
which is \eqref{eq:absorb}.

Now write $h(\Lambda)H_T=h(\Lambda)GD\sqrt T+h(\Lambda)\sigma DT^{1/p}$, and apply \eqref{eq:absorb} with $a=1/2$ to the first term and with $a=a_p$ to the second:
\[
h(\Lambda)H_T\le\kappa(1/2)\Lambda^{1/2}GD\sqrt T+\kappa(a_p)\Lambda^{a_p}\sigma DT^{1/p}=\kappa(1/2)A_T(\Lambda)+\kappa(a_p)B_T(\Lambda).
\]
Inserting this bound into the bound of part~(b) gives $\E R_T(u_{1:T})\le(3+\kappa(1/2))A_T(\Lambda)+(3+\kappa(a_p))B_T(\Lambda)$, which together with $\E R_T(u_{1:T})\le C_T$ is \eqref{eq:logfree-main}. The value $C_0=3+\kappa(1/2)<9.3$ follows from evaluating \eqref{eq:kappa} at $a=1/2$. Finally, \eqref{eq:kappa} shows that $\kappa(a)=O(1+\log(1/a))$ for $a\in(0,1/2]$, and $1/a_p=p/(p-1)$, so $C_p=O(1+\log(p/(p-1)))$.
\end{proof}

\begin{proof}[Proof of Corollary~\ref{cor:gap}]
Write $\Lambda=\Lambda_T$, $A=A_T(\Lambda)$, and $B=B_T(\Lambda)$. Since $\Lambda\ge1$, we have $G D\sqrt T\le A$ and $\sigma DT^{1/p}\le B$, so $H_T\le A+B$. Part~(b) of Theorem~\ref{thm:main} therefore gives
\[
\E R_T(u_{1:T})\le\min\bigl\{(3+h(\Lambda))(A+B),\,C_T\bigr\}\le(3+h(\Lambda))\min\{A+B,\,C_T\},
\]
where the last step uses $3+h(\Lambda)\ge1$. The function $\kappa$ in \eqref{eq:kappa} is decreasing and $a_p\le1/2$, so $C_0=3+\kappa(1/2)\le3+\kappa(a_p)=C_p$. Part~(a) of Theorem~\ref{thm:main} therefore gives
\[
\E R_T(u_{1:T})\le\min\bigl\{C_p(A+B),\,C_T\bigr\}\le C_p\min\{A+B,\,C_T\},
\]
where the last step uses $C_p\ge1$. The smaller of the two bounds is \eqref{eq:gap-to-minimax}.
\end{proof}

\begin{proof}[Proof of Corollary~\ref{cor:static}]
Apply \eqref{eq:pathwise-expected} with $i=N$. Since $L_N=2^N\ge T$, the partition $\mathcal B(L_N)$ consists of the single block $\{1,\ldots,T\}$, and a constant comparator has $P_B=0$ on it. The first term of \eqref{eq:pathwise-expected} is therefore $\sqrt2D\,\E\sqrt{V_{[T]}}$. Since $\pi_N\ge1/2$, \eqref{eq:range} gives $b_N\le\frac{\sqrt5}{2}(1+\log2)$. Hence
\[
\E R_T(u,\ldots,u)\le\Bigl[\sqrt2+\frac{\sqrt5}{2}(1+\log2)\Bigr]D\,\E\sqrt{V_{[T]}}\le\Bigl[\sqrt2+\frac{\sqrt5}{2}(1+\log2)\Bigr]H_T,
\]
where the last step is Lemma~\ref{lem:moment} with $B=\{1,\ldots,T\}$. Proposition~\ref{prop:linearization} gives $\E R_T\le C_T$, which completes the proof.
\end{proof}

%% file: sections/06_lower_bounds.tex
\section{Lower Bounds}
\label{sec:lower}
\begin{theorem}[Minimax regret]
\label{thm:minimax}
Let $\X\subset\R^d$ be nonempty, closed, and convex with diameter $D>0$. Fix $G>0$, $\sigma\ge0$, $p\in(1,2]$, $T\ge2$, and $P\in[0,D(T-1)]$, and let $\Lambda=1+P/D$. Then
\begin{equation}
\label{eq:minimax-main}
\frac1{20}\min\bigl\{A_T(\Lambda)+B_T(\Lambda),\,C_T\bigr\}
\le\Rate(\X;G,\sigma,p,P)
\le3\min\bigl\{A_T(\Lambda)+B_T(\Lambda),\,C_T\bigr\}.
\end{equation}
Moreover, Algorithm~\ref{alg:main}, which uses none of $G$, $\sigma$, $p$, and $P$, satisfies
\begin{equation}
\label{eq:alg-vs-minimax}
\sup_{(\ell_{1:T},\mathcal O)\in\mathcal C_\X(G,\sigma,p)}\ \sup_{u_{1:T}\in\mathcal U(P)}\E R_T(u_{1:T})\le20\min\bigl\{3+h(\Lambda),\,C_p\bigr\}\,\Rate(\X;G,\sigma,p,P),
\end{equation}
with $h$ as in \eqref{eq:h} and $C_p$ as in \eqref{eq:constants}.
\end{theorem}

The upper bound in \eqref{eq:minimax-main} is attained by a single restarted expert that is told $P$ but not $G$, $\sigma$, or $p$, so not knowing the noise parameters costs only a numerical factor. The remaining factor $\min\{3+h(\Lambda),C_p\}$ in \eqref{eq:alg-vs-minimax} is what Algorithm~\ref{alg:main} pays, in our analysis, for not knowing $P$. The lower bound in \eqref{eq:minimax-main} combines a construction with exact gradients (Lemma~\ref{lem:lower-det}) and a construction with heavy-tailed noise (Theorem~\ref{thm:lower-noise}); the proof of the theorem is given at the end of this section. Both constructions use a direction along which $\X$ has width $D$.

\begin{lemma}[A coordinate attaining the diameter]
\label{lem:diameter}
There exist points $a$ and $b$ in $\X$ such that $\norm{b-a}=D$. Define
\begin{equation}
\label{eq:diameter-coordinate}
c=\frac{a+b}{2},\qquad v=\frac{b-a}{D},\qquad z(x)=\inner{v,x-c},
\end{equation}
Then $c\in\X$, $\norm v=1$, $z(a)=-D/2$, $z(b)=D/2$, and $z(x)\in[-D/2,D/2]$ for all $x\in\X$. For $\theta\in\R$, the linear loss $x\mapsto\theta z(x)$ has gradient $\theta v$, of norm $|\theta|$, and for $\theta\ne0$ it is minimized over $\X$ at $c-(D/2)\sgn(\theta)v$, which is $a$ if $\theta>0$ and $b$ if $\theta<0$.
\end{lemma}

The proof is given in Appendix~\ref{app:lower-det}.

In both constructions below, the losses are multiples of $z$ and the oracle outputs are multiples of $v$ whose distributions do not depend on the query point. A point $x\in\X$ therefore affects the losses only through $z(x)\in[-D/2,D/2]$, and it does not affect the observations.

Both constructions also use the integers
\begin{equation}
\label{eq:lower-blocks}
K=\lfloor P/D\rfloor+1,\qquad L=\lfloor T/K\rfloor.
\end{equation}
Since $P\le D(T-1)$, we have $1\le K\le T$ and hence $L\ge1$. Moreover, $K\ge\max\{P/D,1\}\ge(1+P/D)/2=\Lambda/2$, and $L\ge T/(2K)$ because $\lfloor y\rfloor\ge y/2$ for $y\ge1$. Thus
\begin{equation}
\label{eq:rounding}
1\le K\le T,\qquad K\ge\frac\Lambda2,\qquad KL\ge\frac T2.
\end{equation}
The first $KL$ rounds are divided into $K$ blocks $\mathcal I_1,\ldots,\mathcal I_K$ of $L$ consecutive rounds. On the remaining $T-KL$ rounds the losses are zero, the oracle returns zero, and the comparator keeps its last value. A comparator that is constant on each block and takes values in $\{a,b\}$ moves at most $K-1$ times, each time by $D$, so its path length is at most $D(K-1)=D\lfloor P/D\rfloor\le P$.

\begin{theorem}[Lower bound from heavy-tailed noise]
\label{thm:lower-noise}
Under the assumptions of Theorem~\ref{thm:minimax}, for every learner there exist fixed $G$-Lipschitz linear losses, a fixed comparator in $\mathcal U(P)$, and an oracle in $\mathcal C_\X(G,\sigma,p)$ such that
\begin{equation}
\label{eq:noise-lower-main}
\E R_T(u_{1:T})\ge\frac{1-e^{-1/2}}4\min\bigl\{B_T(\Lambda),\,C_T\bigr\}\ge\frac1{12}\min\bigl\{B_T(\Lambda),\,C_T\bigr\}.
\end{equation}
The oracle satisfies the stronger conditional bound $\E[\norm{\epsilon_t}^p\mid\F_{t-1}]\le\sigma^p$ for every learner.
\end{theorem}

\begin{proof}
If $\sigma=0$ the right side of \eqref{eq:noise-lower-main} is zero and there is nothing to prove, so let $\sigma>0$. Draw independent signs $\nu_1,\ldots,\nu_K$, each uniform on $\{-1,+1\}$. Each realization of the signs defines one instance; we bound the regret averaged over the signs and then select one realization. Let
\begin{equation}
\label{eq:oracle-parameters}
q=\frac1{2L},\qquad Z=\min\bigl\{\sigma L^{1/p},\,2GL\bigr\},
\end{equation}
and for every round $t$ in block $\mathcal I_k$ set
\begin{equation}
\label{eq:hard-loss}
\ell_t(x)=qZ\nu_k\,z(x),\qquad u_t=c-\frac D2\nu_kv.
\end{equation}
By Lemma~\ref{lem:diameter}, $\ell_t$ is linear with gradient $qZ\nu_kv$, whose norm $qZ\le G$ because $Z\le2GL$. The same lemma shows that $u_t\in\{a,b\}$ minimizes $\ell_t$ over $\X$, and the comparator is constant on each block, so its path length is at most $P$.

The oracle uses independent Bernoulli variables $\beta_t$ with $\Pr(\beta_t=1)=q$, drawn independently of the signs and of the learner's seed, and for $t\in\mathcal I_k$ it returns
\begin{equation}
\label{eq:hard-oracle}
g_t=\nu_kZ\beta_tv.
\end{equation}
Fix a realization of the signs and any learner. Since $\beta_t$ is independent of $\F_{t-1}$, $\E[g_t\mid\F_{t-1}]=qZ\nu_kv$, which is the gradient of $\ell_t$ at every point, so \eqref{eq:oracle} holds. The noise is $\epsilon_t=\nu_kZ(\beta_t-q)v$, with norm $Z|\beta_t-q|$, so
\begin{equation}
\label{eq:hard-moment}
\E\bigl[\norm{\epsilon_t}^p\mid\F_{t-1}\bigr]=Z^p\bigl[q(1-q)^p+(1-q)q^p\bigr]\le2qZ^p=\frac{Z^p}{L}\le\sigma^p,
\end{equation}
where we used $(1-q)^p\le1$, $q^p\le q$, and $Z\le\sigma L^{1/p}$. On the padding rounds the losses, the oracle outputs, and the noise are zero. Hence every realization of the signs defines a loss sequence and an oracle in $\mathcal C_\X(G,\sigma,p)$ that satisfy the conditional moment bound.

Now fix a block $k$ and its $j$-th round $t$, where $1\le j\le L$. Let $E_{k,j}$ be the event that the first $j-1$ Bernoulli variables of block $k$ are all zero, so that
\begin{equation}
\label{eq:survival-prob}
\Pr(E_{k,j})=(1-q)^{j-1}.
\end{equation}
On $E_{k,j}$, every observation in block $k$ before round $t$ is zero, so $x_t$ coincides with a random variable $\tilde x_t$ that is a function of the seed and of the observations made before block $k$. These observations depend only on $\nu_1,\ldots,\nu_{k-1}$ and on Bernoulli variables outside block $k$. Hence the sign $\nu_k$ is independent of the pair $(\tilde x_t,\mathbf 1_{E_{k,j}})$, and
\[
\E\bigl[\nu_kz(x_t)\mathbf 1_{E_{k,j}}\bigr]=\E[\nu_k]\,\E\bigl[z(\tilde x_t)\mathbf 1_{E_{k,j}}\bigr]=0.
\]
Since $z(u_t)=-\nu_kD/2$, the regret at round $t$ is
\begin{equation}
\label{eq:round-noise-regret}
\ell_t(x_t)-\ell_t(u_t)=qZ\Bigl(\nu_kz(x_t)+\frac D2\Bigr)\ge0,
\end{equation}
where nonnegativity follows from $|z(x_t)|\le D/2$. Dropping the nonnegative contribution of the complement of $E_{k,j}$ and using the last two displays together with \eqref{eq:survival-prob},
\[
\E\bigl[\ell_t(x_t)-\ell_t(u_t)\bigr]\ge\E\Bigl[qZ\Bigl(\nu_kz(x_t)+\frac D2\Bigr)\mathbf 1_{E_{k,j}}\Bigr]=\frac{qZD}{2}(1-q)^{j-1}.
\]
Summing over the $L$ rounds of block $k$ and using $(1-q)^L\le e^{-qL}=e^{-1/2}$,
\begin{equation}
\label{eq:block-noise-lower}
\E\sum_{t\in\mathcal I_k}\bigl(\ell_t(x_t)-\ell_t(u_t)\bigr)\ge\frac{qZD}{2}\sum_{j=1}^L(1-q)^{j-1}=\frac{ZD}{2}\bigl[1-(1-q)^L\bigr]\ge\frac{ZD}{2}\bigl(1-e^{-1/2}\bigr).
\end{equation}

The padding rounds contribute zero regret, so summing \eqref{eq:block-noise-lower} over the $K$ blocks gives
\begin{equation}
\label{eq:finite-block-noise}
\E R_T(u_{1:T})\ge\frac{D}{2}\bigl(1-e^{-1/2}\bigr)KZ=\frac{D}{2}\bigl(1-e^{-1/2}\bigr)\min\bigl\{\sigma KL^{1/p},\,2GKL\bigr\}.
\end{equation}
By \eqref{eq:rounding},
\begin{equation}
\label{eq:noise-rounding}
KL^{1/p}=(KL)^{1/p}K^{(p-1)/p}\ge\Bigl(\frac T2\Bigr)^{1/p}\Bigl(\frac\Lambda2\Bigr)^{(p-1)/p}=\frac12T^{1/p}\Lambda^{(p-1)/p},
\qquad 2GKL\ge GT.
\end{equation}
Substituting \eqref{eq:noise-rounding} into \eqref{eq:finite-block-noise} and using $\min\{y/2,y'\}\ge\frac12\min\{y,y'\}$ gives the first inequality in \eqref{eq:noise-lower-main}. The second follows from $e^{1/2}\ge3/2$, which gives $1-e^{-1/2}\ge1/3$.

The expectation in \eqref{eq:finite-block-noise} averages over the $2^K$ equally likely sign sequences, so at least one sign sequence has expected regret at least this average. Fixing that sequence makes the losses and the comparator deterministic, and the remaining expectation is over the Bernoulli variables and the learner's seed, as required.
\end{proof}

The construction relies on the learner observing only gradients. An oracle that also returned the value $\ell_t(x_t)=qZ\nu_kz(x_t)$ would reveal $\nu_k$ at the first round of block $k$ with $z(x_t)\ne0$. Because the oracle satisfies a conditional moment bound, the lower bound also holds for the smaller class in which $\E[\norm{\epsilon_t}^p\mid\F_{t-1}]\le\sigma^p$ for every $t$. The oracle takes only two values in each block: the lower bound comes from rare large observations, and infinite variance is not needed for a lower bound over the class defined by a finite $p$-th moment.

The deterministic part of the lower bound is the blocking argument of \citet{zhang2018adaptive}, adapted to an arbitrary compact convex domain through Lemma~\ref{lem:diameter}, to an arbitrary real budget $P$ through \eqref{eq:lower-blocks}, and to an explicit constant.

\begin{lemma}[Lower bound with exact gradients]
\label{lem:lower-det}
Under the assumptions of Theorem~\ref{thm:minimax}, for every learner there exist fixed $G$-Lipschitz linear losses, observed with exact gradients, and a fixed comparator in $\mathcal U(P)$ such that
\begin{equation}
\label{eq:det-lower}
\E R_T(u_{1:T})\ge\frac{GD}{4\sqrt3}\sqrt{T\Lambda}\ge\frac18A_T(\Lambda).
\end{equation}
\end{lemma}

The proof is given in Appendix~\ref{app:lower-det}, together with the proof of Lemma~\ref{lem:diameter}.

\begin{proof}[Proof of Theorem~\ref{thm:minimax}]
Write $A=A_T(\Lambda)$, $B=B_T(\Lambda)$, and $\mathcal R=\Rate(\X;G,\sigma,p,P)$.

\textit{Lower bound.} The instances in Lemma~\ref{lem:lower-det} and Theorem~\ref{thm:lower-noise} belong to $\mathcal C_\X(G,\sigma,p)$ and their comparators belong to $\mathcal U(P)$, so for every learner the supremum in \eqref{eq:minimax-definition} is at least each of the two bounds. Hence
\begin{equation}
\label{eq:combined-lower}
\mathcal R\ge\max\Bigl\{\frac A8,\ \frac1{12}\min\{B,C_T\}\Bigr\}.
\end{equation}
If $B\ge C_T$, then \eqref{eq:combined-lower} gives $\mathcal R\ge C_T/12\ge\min\{A+B,C_T\}/20$. If $B<C_T$, then \eqref{eq:combined-lower} gives $8\mathcal R\ge A$ and $12\mathcal R\ge B$, so $20\mathcal R\ge A+B\ge\min\{A+B,C_T\}$. In both cases the lower bound in \eqref{eq:minimax-main} holds.

\textit{Upper bound.} Let $L=2^{\lfloor\log_2(T/\Lambda)\rfloor}$ and consider the learner that runs a single restarted expert with block length $L$, that is, the update \eqref{eq:expert} with restarts at rounds $1,1+L,1+2L,\ldots$, and plays the expert's point. Fix an instance in $\mathcal C_\X(G,\sigma,p)$ and a comparator in $\mathcal U(P)$. Summing the second inequality of Lemma~\ref{lem:block}, with $\mathcal K=\X$ and first point $x_0$ in every block, over the blocks of $\mathcal B(L)$ and applying Proposition~\ref{prop:linearization} gives
\[
\E R_T(u_{1:T})\le\sqrt2\sum_{I\in\mathcal B(L)}(D+P_I)\,\E\sqrt{V_I}.
\]
The comparator satisfies $P_T\le P=D(\Lambda-1)$, so Lemma~\ref{lem:scale} bounds the right side by $3(A+B)$. Proposition~\ref{prop:linearization} also gives $R_T\le C_T$, so the expected regret of this learner is at most $\min\{3(A+B),C_T\}\le3\min\{A+B,C_T\}$ on every instance and comparator. Taking the infimum over learners proves the upper bound in \eqref{eq:minimax-main}.

\textit{Comparison with Algorithm~\ref{alg:main}.} Let $(\ell_{1:T},\mathcal O)\in\mathcal C_\X(G,\sigma,p)$ and $u_{1:T}\in\mathcal U(P)$. The oracle satisfies \eqref{eq:oracle} and \eqref{eq:moment-assumption} along the trajectory of Algorithm~\ref{alg:main}, and $\Lambda_T\le\Lambda$. The right side of \eqref{eq:gap-to-minimax} is nondecreasing in $\Lambda_T$, because $h$, $A_T$, and $B_T$ are, so Corollary~\ref{cor:gap} gives
\[
\E R_T(u_{1:T})\le\min\{3+h(\Lambda),\,C_p\}\cdot\min\{A+B,\,C_T\}.
\]
The lower bound in \eqref{eq:minimax-main} gives $\min\{A+B,C_T\}\le20\mathcal R$, which proves \eqref{eq:alg-vs-minimax}.
\end{proof}

Dividing \eqref{eq:minimax-main} by the ceiling $C_T$ shows that, for $G>0$,
\begin{equation}
\label{eq:normalized-rate}
\frac{\Rate(\X;G,\sigma,p,P)}{GDT}=\Theta\Biggl(\min\Biggl\{\sqrt{\frac\Lambda T}+\frac\sigma G\Bigl(\frac\Lambda T\Bigr)^{(p-1)/p},\ 1\Biggr\}\Biggr)
\end{equation}
with numerical constants. For fixed $G$, $D$, $\sigma$, and $p$, the minimax regret is therefore sublinear in $T$ exactly when $\Lambda=o(T)$, and by Theorem~\ref{thm:main} the same holds for Algorithm~\ref{alg:main}. When $\sigma>0$, the noise term in \eqref{eq:normalized-rate} reaches the ceiling once $\Lambda\ge T(G/\sigma)^{p/(p-1)}$; if this threshold is below one, that is, if $\sigma^p>G^pT^{p-1}$, the noise alone forces linear regret even for a fixed comparator.

%% file: sections/07_conclusion.tex
\section{Conclusion}
\label{sec:conclusion}
With one stochastic subgradient per round, no loss values, and noise with an unconditional $p$-th moment bound, the minimax expected dynamic regret on a compact convex domain of positive diameter is of order $\min\{GD\sqrt{T\Lambda_T}+\sigma DT^{1/p}\Lambda_T^{(p-1)/p},\,GDT\}$, with numerical constants that do not depend on $p$. Algorithm~\ref{alg:main} attains this rate for each fixed $p>1$ without knowing $G$, $\sigma$, $p$, or the path length. Our analysis leaves the factor $\min\{3+h(\Lambda_T),C_p\}$ relative to the benchmark rate, which is of order $\min\{1+\log(1+\log\Lambda_T),\,1+\log(p/(p-1))\}$; a single restarted expert that is told the path-length budget, but not $G$, $\sigma$, or $p$, needs no such factor. Whether an algorithm that does not know the path length must pay a factor that grows with $\Lambda_T$, or as $p\to1$, is open. Under the unconditional moment bound \eqref{eq:moment-assumption}, Appendix~\ref{app:high-probability} shows that, up to the ceiling $GDT$, the noise contribution to a bound that holds with probability $1-\delta$ must be of order $\sigma DT^{1/p}\Lambda_T^{(p-1)/p}\delta^{-1/p}$, and Algorithm~\ref{alg:main} attains this order. The corresponding confidence question under uniformly bounded conditional moments is not settled by Proposition~\ref{prop:hp-lower}. Finally, the analysis uses the diameter both in the stepsizes and to bound distances between the learner and the comparator, so it does not apply to unbounded domains.

%% file: sections/A_one_projection.tex
\section[Proof of the one-projection reduction]{Proof of the Inequalities in Remark~\ref{rem:one-projection}}
\label{app:one-projection}
The reduction and the vector $\tilde g_t$ are those of \citet{zhao2025efficient}. The adaptation concerns the feedback: $g_t$ is a stochastic vector that need not be bounded, and its conditional mean is a subgradient only at $x_t$. The two inequalities below hold for every realization, so the pathwise analysis of Section~\ref{sec:upper} applies to $\tilde g_t$ without any assumption on its distribution.

\begin{proof}
Fix a round $t$. If $y_t=x_t$, then $\tilde g_t=g_t$ and $y_t-u=x_t-u$ for every $u$, so both inequalities hold with equality. Assume $y_t\ne x_t$, so that $v_t$ is a unit vector, and let $\mu_t=\min\{0,\inner{g_t,v_t}\}$, so that $\tilde g_t=g_t-\mu_tv_t$.

We first compare the norms. If $\inner{g_t,v_t}\ge0$, then $\mu_t=0$ and $\tilde g_t=g_t$. If $\inner{g_t,v_t}<0$, then $\tilde g_t=g_t-\inner{g_t,v_t}v_t$, and $\norm{v_t}=1$ gives $\inner{\tilde g_t,v_t}=0$. Hence $g_t=\tilde g_t+\inner{g_t,v_t}v_t$ is a sum of two orthogonal vectors, and
\begin{equation}
\label{eq:op-norm}
\norm{g_t}^2=\norm{\tilde g_t}^2+\inner{g_t,v_t}^2\ge\norm{\tilde g_t}^2.
\end{equation}

We next compare the linear losses. Let $u\in\X$. Since $x_t=\Pi_\X(y_t)$ and $u\in\X$, the projection inequality gives $\inner{y_t-x_t,u-x_t}\le0$, and dividing by $\norm{y_t-x_t}>0$,
\begin{equation}
\label{eq:op-projection}
\inner{v_t,x_t-u}\ge0.
\end{equation}
Write $y_t-u=\norm{y_t-x_t}v_t+(x_t-u)$. Then $\inner{g_t,y_t-u}=\norm{y_t-x_t}\inner{g_t,v_t}+\inner{g_t,x_t-u}$ and, since $\norm{v_t}=1$, $\inner{v_t,y_t-u}=\norm{y_t-x_t}+\inner{v_t,x_t-u}$. Subtracting $\mu_t$ times the second identity from the first,
\begin{equation}
\label{eq:op-identity}
\inner{\tilde g_t,y_t-u}=\inner{g_t,x_t-u}+\norm{y_t-x_t}\inner{g_t,v_t}-\mu_t\bigl(\norm{y_t-x_t}+\inner{v_t,x_t-u}\bigr).
\end{equation}
If $\inner{g_t,v_t}\ge0$, then $\mu_t=0$, the second term on the right of \eqref{eq:op-identity} is nonnegative, and the third vanishes, so $\inner{\tilde g_t,y_t-u}\ge\inner{g_t,x_t-u}$. If $\inner{g_t,v_t}<0$, then $\mu_t=\inner{g_t,v_t}$, the second term cancels the part $-\mu_t\norm{y_t-x_t}$ of the third, and \eqref{eq:op-identity} becomes
\[
\inner{\tilde g_t,y_t-u}=\inner{g_t,x_t-u}-\inner{g_t,v_t}\inner{v_t,x_t-u}\ge\inner{g_t,x_t-u},
\]
because $-\inner{g_t,v_t}>0$ and $\inner{v_t,x_t-u}\ge0$ by \eqref{eq:op-projection}. In both cases
\begin{equation}
\label{eq:op-comparison}
\inner{g_t,x_t-u}\le\inner{\tilde g_t,y_t-u}\qquad\text{for every }u\in\X.
\end{equation}
Taking $u=u_t$ in \eqref{eq:op-comparison} and summing over $t$ gives $\sum_t\inner{g_t,x_t-u_t}\le\sum_t\inner{\tilde g_t,y_t-u_t}$.

It remains to explain why the upper bounds hold with $2D$ in place of $D$. The points $y_t$ are produced by Algorithm~\ref{alg:main} run on $\mathcal Y$, whose diameter is $2D$, with the vectors $\tilde g_t$ in place of $g_t$. Lemma~\ref{lem:block}, with $\mathcal K=\mathcal Y$, and Lemma~\ref{lem:master} hold for arbitrary vectors, and every comparator in $\X$ lies in $\mathcal Y$, so both inequalities of \eqref{eq:pathwise} hold for $\sum_t\inner{\tilde g_t,y_t-u_t}$ with $2D$ in place of $D$ and with the energies of $\tilde g_t$. By \eqref{eq:op-norm}, these energies are at most the corresponding energies of $g_t$. The point $x_t=\Pi_\X(y_t)$ is $\F_{t-1}$-measurable because $y_t$ is and projection is continuous, so Proposition~\ref{prop:linearization} applies to $x_t$, and together with \eqref{eq:op-comparison} it gives $\E R_T(u_{1:T})\le\E\sum_t\inner{\tilde g_t,y_t-u_t}$. From here the proof of Theorem~\ref{thm:main} applies verbatim with $2D$ in place of $D$; the normalized path length becomes $1+P_T/(2D)\le\Lambda_T$, and the rates are nondecreasing in the normalized path length and linear in the diameter, so the resulting bounds are at most twice those of Theorem~\ref{thm:main}.
\end{proof}

%% file: sections/B_auxiliary_proofs.tex
\section[Proofs of the expert-level results]{Proofs of Lemma~\ref{lem:block}, Proposition~\ref{prop:pathwise}, and Lemma~\ref{lem:moment}}
\label{app:aux}
The three results proved here adapt established arguments, which are cited in each subsection.

\subsection[Proof of the restarted-expert bound]{Proof of Lemma~\ref{lem:block}}
The argument combines the moving-comparator analysis of projected gradient steps \citep{zinkevich2003online,hall2015online} with the AdaGrad stepsize and summation step \citep{duchi2011adaptive}. The adaptation is the regrouping in \eqref{eq:gamma-split}, which charges each movement of the comparator to the energy accumulated before it.

\begin{proof}
Let $n=e-s+1$ and re-index the rounds of the block as $1,\ldots,n$, keeping the notation $y_t$, $g_t$, and $u_t$ for the local rounds, and let $V_t=\sum_{r=1}^t\norm{g_r}^2$ with $V_0=0$. Then $y_1\in\mathcal K$ is arbitrary, and with $\eta=D/\sqrt2$ the update reads $y_{t+1}=\Pi_{\mathcal K}(y_t-\eta g_t/\sqrt{V_t})$ if $V_t>0$ and $y_{t+1}=y_t$ if $V_t=0$; we also define $y_{n+1}$ by the same rule. Set $a_t=\sqrt{V_t}/(2\eta)$ for $t\ge0$. If $n=1$, every sum over $t=2,\ldots,n$ below is empty and equals zero. The point $y_{n+1}$ appears only in the term $-a_n\norm{y_{n+1}-u_n}^2$, which is dropped below.

Fix $t$ with $V_t>0$. Since $u_t\in\mathcal K$ and projection onto a closed convex set does not increase the distance to any point of the set,
\[
\norm{y_{t+1}-u_t}^2\le\Bigl\|y_t-\frac{\eta g_t}{\sqrt{V_t}}-u_t\Bigr\|^2
=\norm{y_t-u_t}^2-\frac{2\eta}{\sqrt{V_t}}\inner{g_t,y_t-u_t}+\frac{\eta^2\norm{g_t}^2}{V_t}.
\]
Multiplying by $\sqrt{V_t}/(2\eta)$ and rearranging gives the inequality underlying the moving-comparator analysis of \citet{zinkevich2003online},
\begin{equation}
\label{eq:onestep}
\inner{g_t,y_t-u_t}\le a_t\bigl(\norm{y_t-u_t}^2-\norm{y_{t+1}-u_t}^2\bigr)+\frac{\eta\norm{g_t}^2}{2\sqrt{V_t}}.
\end{equation}
If $V_t=0$, then $g_t=0$, $a_t=0$, and $y_{t+1}=y_t$, so \eqref{eq:onestep} holds with both sides equal to zero once its last term is read as zero. Summing \eqref{eq:onestep} over $t=1,\ldots,n$ gives
\begin{equation}
\label{eq:block-sum}
\sum_{t=1}^n\inner{g_t,y_t-u_t}\le\Gamma+\sum_{t:\,V_t>0}\frac{\eta\norm{g_t}^2}{2\sqrt{V_t}},
\qquad\text{where }\ 
\Gamma=\sum_{t=1}^na_t\bigl(\norm{y_t-u_t}^2-\norm{y_{t+1}-u_t}^2\bigr).
\end{equation}

We bound the two terms on the right of \eqref{eq:block-sum} in turn. In $\Gamma$, shift the index of the negative terms by one and drop the nonpositive term $-a_n\norm{y_{n+1}-u_n}^2$; adding and subtracting $a_{t-1}\norm{y_t-u_t}^2$ then gives
\begin{align}
\Gamma&\le a_1\norm{y_1-u_1}^2+\sum_{t=2}^n\bigl(a_t\norm{y_t-u_t}^2-a_{t-1}\norm{y_t-u_{t-1}}^2\bigr)\notag\\
&=a_1\norm{y_1-u_1}^2+\sum_{t=2}^n(a_t-a_{t-1})\norm{y_t-u_t}^2+\sum_{t=2}^na_{t-1}\bigl(\norm{y_t-u_t}^2-\norm{y_t-u_{t-1}}^2\bigr).
\label{eq:gamma-split}
\end{align}
The sequence $a_t$ is nondecreasing and every distance between two points of $\mathcal K$ is at most $D$, so the first two terms on the right of \eqref{eq:gamma-split} add up to at most $D^2\bigl(a_1+\sum_{t=2}^n(a_t-a_{t-1})\bigr)=D^2a_n$. In the last sum, write the difference of squares as $\bigl(\norm{y_t-u_t}-\norm{y_t-u_{t-1}}\bigr)\bigl(\norm{y_t-u_t}+\norm{y_t-u_{t-1}}\bigr)$. The first factor is at most $\norm{u_t-u_{t-1}}$ by the triangle inequality, and the second factor lies in $[0,2D]$. If the first factor is negative the product is negative, and otherwise the product is at most $2D\norm{u_t-u_{t-1}}$; in both cases
\[
\norm{y_t-u_t}^2-\norm{y_t-u_{t-1}}^2\le2D\norm{u_t-u_{t-1}}.
\]
Substituting these two estimates into \eqref{eq:gamma-split} and using $a_t=\sqrt{V_t}/(2\eta)$, we obtain
\begin{equation}
\label{eq:gamma-bound}
\Gamma\le\frac{D^2}{2\eta}\sqrt{V_n}+\frac{D}{\eta}\sum_{t=2}^n\sqrt{V_{t-1}}\,\norm{u_t-u_{t-1}}.
\end{equation}

For the second term on the right of \eqref{eq:block-sum}, take $t$ with $V_t>0$. Then $\norm{g_t}^2=V_t-V_{t-1}$ and $\sqrt{V_{t-1}}\le\sqrt{V_t}$, hence
\[
\frac{\norm{g_t}^2}{\sqrt{V_t}}=\bigl(\sqrt{V_t}-\sqrt{V_{t-1}}\bigr)\,\frac{\sqrt{V_t}+\sqrt{V_{t-1}}}{\sqrt{V_t}}\le2\bigl(\sqrt{V_t}-\sqrt{V_{t-1}}\bigr).
\]
If $V_t=0$, then also $\sqrt{V_t}-\sqrt{V_{t-1}}=0$. Summing over $t$ and telescoping, as in the analysis of AdaGrad \citep{duchi2011adaptive},
\begin{equation}
\label{eq:adagrad-sum}
\sum_{t:\,V_t>0}\frac{\eta\norm{g_t}^2}{2\sqrt{V_t}}\le\eta\sum_{t=1}^n\bigl(\sqrt{V_t}-\sqrt{V_{t-1}}\bigr)=\eta\sqrt{V_n}.
\end{equation}

Inserting \eqref{eq:gamma-bound} and \eqref{eq:adagrad-sum} into \eqref{eq:block-sum}, the coefficient of $\sqrt{V_n}$ is $D^2/(2\eta)+\eta=\sqrt2D$ and the coefficient of the movement sum is $D/\eta=\sqrt2$, because $\eta=D/\sqrt2$. In the original indexing $V_n=V_B$ and $V_{t-1}=V_{B,<t}$, which proves the first inequality in \eqref{eq:block-refined}. The second inequality follows from $V_{B,<t}\le V_B$ and $\sum_{t=s+1}^e\norm{u_t-u_{t-1}}=P_B$.
\end{proof}

\subsection[Proof of the comparison with each expert]{Proof of Proposition~\ref{prop:pathwise}}
The decomposition into a master term and an expert term on common linear losses is that of \citet{zhang2018adaptive}.

\begin{proof}
Fix $i$. By \eqref{eq:mixture}, $m_{t,i}=\inner{g_t,x_t^{(i)}-x_t}$ and $\inner{w_t,m_t}=0$, so for every round
\[
\inner{g_t,x_t-u_t}=\bigl(\inner{w_t,m_t}-m_{t,i}\bigr)+\inner{g_t,x_t^{(i)}-u_t}.
\]
Summing over $t$,
\begin{equation}
\label{eq:decomposition}
\sum_{t=1}^T\inner{g_t,x_t-u_t}=\sum_{t=1}^T\bigl(\inner{w_t,m_t}-m_{t,i}\bigr)+\sum_{B\in\mathcal B(L_i)}\sum_{t\in B}\inner{g_t,x_t^{(i)}-u_t}.
\end{equation}
Lemma~\ref{lem:master} bounds the first sum on the right of \eqref{eq:decomposition} by $b_i(\sum_tr_t^2)^{1/2}$. On each block $B\in\mathcal B(L_i)$, the points $x_t^{(i)}$ are those of a restarted expert, so the first inequality of Lemma~\ref{lem:block} bounds the inner sum by $\sqrt2$ times the bracket in \eqref{eq:expert-cost}; summing over the blocks bounds the second sum by $\sqrt2\,\mathcal E_i(u_{1:T})$. This proves the first inequality in \eqref{eq:pathwise}. For the second, $V_{B,<t}\le V_B$ gives $\mathcal E_i(u_{1:T})\le\sum_{B\in\mathcal B(L_i)}(D+P_B)\sqrt{V_B}$, and \eqref{eq:range-geometric} gives $\sum_tr_t^2\le D^2V_{[T]}$.

For \eqref{eq:pathwise-expected}, fix the comparator, so that every $P_B$ is a number. The right side of \eqref{eq:pathwise} is integrable: for any set of rounds $B$ we have $\sqrt{V_B}\le\sum_{t\in B}\norm{g_t}\le\sum_{t\in B}\bigl(G+\norm{\epsilon_t}\bigr)$, and Jensen's inequality with \eqref{eq:moment-assumption} gives $\E\norm{\epsilon_t}\le(\E\norm{\epsilon_t}^p)^{1/p}\le\sigma$. Taking expectations in \eqref{eq:pathwise} and applying the first inequality of Proposition~\ref{prop:linearization}, which holds because $x_t$ is $\F_{t-1}$-measurable (Section~\ref{sec:algorithm}), proves \eqref{eq:pathwise-expected}.
\end{proof}

\subsection[Proof of the blockwise energy bound]{Proof of Lemma~\ref{lem:moment}}
The argument is that of \citet{liu2025heavy}, written for an arbitrary deterministic set of rounds $B$ and using only the unconditional bound \eqref{eq:moment-assumption}.
\begin{proof}
Let $n=|B|$. The vector $(g_t)_{t\in B}\in\R^{dn}$ is the sum of $(s_t(x_t))_{t\in B}$ and $(\epsilon_t)_{t\in B}$, so the triangle inequality for the Euclidean norm on $\R^{dn}$ gives
\begin{equation}
\label{eq:energy-triangle}
\Bigl(\sum_{t\in B}\norm{g_t}^2\Bigr)^{1/2}\le\Bigl(\sum_{t\in B}\norm{s_t(x_t)}^2\Bigr)^{1/2}+\Bigl(\sum_{t\in B}\norm{\epsilon_t}^2\Bigr)^{1/2}.
\end{equation}
The first term on the right is at most $G\sqrt n$ because $\norm{s_t(x_t)}\le G$.

For the second term, we use that $(\sum_tc_t^2)^{1/2}\le(\sum_tc_t^p)^{1/p}$ for nonnegative numbers $c_t$ and $1\le p\le2$. To see this, note that both sides are positively homogeneous of degree one in $(c_t)$, so it suffices to treat the case $\sum_tc_t^p=1$; then every $c_t\le1$, hence $c_t^2\le c_t^p$, and $\sum_tc_t^2\le1$. With $c_t=\norm{\epsilon_t}$,
\begin{equation}
\label{eq:norm-comparison}
\Bigl(\sum_{t\in B}\norm{\epsilon_t}^2\Bigr)^{1/2}\le\Bigl(\sum_{t\in B}\norm{\epsilon_t}^p\Bigr)^{1/p}.
\end{equation}
The function $y\mapsto y^{1/p}$ is concave on $[0,\infty)$. Since $B$ is deterministic, $\E\sum_{t\in B}\norm{\epsilon_t}^p=\sum_{t\in B}\E\norm{\epsilon_t}^p\le n\sigma^p$ by \eqref{eq:moment-assumption}, and Jensen's inequality gives
\begin{equation}
\label{eq:moment-jensen}
\E\Bigl(\sum_{t\in B}\norm{\epsilon_t}^p\Bigr)^{1/p}\le\Bigl(\E\sum_{t\in B}\norm{\epsilon_t}^p\Bigr)^{1/p}\le\sigma n^{1/p}.
\end{equation}
Taking expectations in \eqref{eq:energy-triangle} and combining \eqref{eq:norm-comparison} with \eqref{eq:moment-jensen} gives \eqref{eq:moment-energy}. Only unconditional expectations were used, so neither independence of the noise nor a conditional moment bound is needed.
\end{proof}

%% file: sections/C_master_proof.tex
\section[Proof of the master regret bound]{Proof of Lemma~\ref{lem:master}}
\label{app:master-proof}
The proof follows the analysis of AdaHedge by \citet{derooij2014follow}. Two changes are needed: the uniform prior is replaced by $\pi$, which enters through the relative entropy to $\pi$ in \eqref{eq:potentials}, and the losses are arbitrary real vectors, which requires the separate treatment of $\lambda_t=0$ and the bound $\delta_t\le r_t$ in place of a known loss range.
\begin{proof}
For $w$ in the probability simplex $\Delta_M$, let $\psi(w)=\KL(w\,\|\,\pi)=\sum_iw_i\log(w_i/\pi_i)$, with $0\log0=0$; then $\psi\ge0$, and $\psi(e_i)=\log(1/\pi_i)$ for the $i$-th unit vector $e_i$. For $t=1,\ldots,T+1$ define
\begin{equation}
\label{eq:potentials}
\Phi_t=\min_{w\in\Delta_M}\bigl\{\inner{w,S_{t-1}}+\lambda_t\psi(w)\bigr\},\qquad
\widetilde\Phi_{t+1}=\min_{w\in\Delta_M}\bigl\{\inner{w,S_t}+\lambda_t\psi(w)\bigr\},
\end{equation}
where $\widetilde\Phi_{t+1}$ is defined for $t\le T$ and uses the new losses with the old parameter $\lambda_t$.

We first show that $\delta_t\ge0$ for every $t$. If $\lambda_t>0$, Jensen's inequality for the convex function $y\mapsto e^{-y/\lambda_t}$ gives $\sum_iw_{t,i}e^{-m_{t,i}/\lambda_t}\ge e^{-\inner{w_t,m_t}/\lambda_t}$; taking logarithms and multiplying by $\lambda_t$ shows that the second case of \eqref{eq:gap} is nonnegative. If $\lambda_t=0$, then $\inner{\pi,m_t}\ge\min_im_{t,i}$. Consequently $\lambda_{t+1}=\lambda_t+\delta_t\ge\lambda_t$, so the sequence $\lambda_1=0\le\lambda_2\le\cdots$ is nondecreasing.

Next we show that, for every $t\le T$,
\begin{equation}
\label{eq:gap-potential}
\delta_t=\inner{w_t,m_t}+\Phi_t-\widetilde\Phi_{t+1}.
\end{equation}
Suppose first that $\lambda_t>0$, and let $q$ be the Gibbs distribution $q_i=\pi_ie^{-S_{t-1,i}/\lambda_t}/Z_t$ with $Z_t=\sum_j\pi_je^{-S_{t-1,j}/\lambda_t}$. Since $\log(w_i/q_i)=\log(w_i/\pi_i)+S_{t-1,i}/\lambda_t+\log Z_t$,
\[
\inner{w,S_{t-1}}+\lambda_t\psi(w)=\lambda_t\KL(w\,\|\,q)-\lambda_t\log Z_t .
\]
The relative entropy is nonnegative and vanishes only at $w=q$, so the minimizer in $\Phi_t$ is $q=w_t$ from \eqref{eq:weights}, and $\Phi_t=-\lambda_t\log Z_t$. The same argument with $S_t=S_{t-1}+m_t$ gives $\widetilde\Phi_{t+1}=-\lambda_t\log\sum_j\pi_je^{-S_{t-1,j}/\lambda_t}e^{-m_{t,j}/\lambda_t}$. Dividing the two sums,
\[
\widetilde\Phi_{t+1}-\Phi_t=-\lambda_t\log\sum_jw_{t,j}e^{-m_{t,j}/\lambda_t},
\]
and \eqref{eq:gap-potential} follows from \eqref{eq:gap}. Suppose next that $\lambda_t=0$. Since the $\lambda_s$ are nondecreasing and nonnegative, $\lambda_s=0$ for all $s\le t$, and since $0=\lambda_t=\sum_{s<t}\delta_s$ is a sum of nonnegative terms, $\delta_s=0$ for all $s<t$. For such $s$, $\delta_s=\inner{\pi,m_s}-\min_im_{s,i}=\sum_i\pi_i(m_{s,i}-\min_jm_{s,j})$, and since every $\pi_i>0$ this vanishes only if all coordinates of $m_s$ are equal. Hence $S_{t-1}=\alpha\mathbf 1$ for some $\alpha\in\R$, so $\Phi_t=\alpha$ and $\widetilde\Phi_{t+1}=\alpha+\min_im_{t,i}$. Since $w_t=\pi$, the right side of \eqref{eq:gap-potential} equals $\inner{\pi,m_t}-\min_im_{t,i}=\delta_t$.

Because $\psi\ge0$ and $\lambda_{t+1}\ge\lambda_t$, increasing the coefficient of $\psi$ in \eqref{eq:potentials} from $\lambda_t$ to $\lambda_{t+1}$ cannot decrease the minimum, so $\widetilde\Phi_{t+1}\le\Phi_{t+1}$. Combining this with \eqref{eq:gap-potential},
\[
\inner{w_t,m_t}=\delta_t-\Phi_t+\widetilde\Phi_{t+1}\le\delta_t+\Phi_{t+1}-\Phi_t .
\]
We sum this inequality over $t=1,\ldots,T$. The potentials telescope, $\Phi_1=\min_w\{\inner{w,0}+0\}=0$, and $\sum_{t=1}^T\delta_t=\lambda_{T+1}$ by \eqref{eq:master-update}. Moreover, for any $u\in\Delta_M$, the definition of $\Phi_{T+1}$ as a minimum gives $\Phi_{T+1}\le\inner{u,S_T}+\lambda_{T+1}\psi(u)$. Since $\inner{u,S_T}=\sum_t\inner{u,m_t}$,
\begin{equation}
\label{eq:master-telescope}
\sum_{t=1}^T\inner{w_t-u,m_t}\le\lambda_{T+1}+\lambda_{T+1}\psi(u)=\bigl(1+\psi(u)\bigr)\lambda_{T+1}.
\end{equation}

It remains to bound $\lambda_{T+1}$. Let $\underline m_t=\min_im_{t,i}$ and $\overline m_t=\max_im_{t,i}$, so $r_t=\overline m_t-\underline m_t$. If $\lambda_t>0$, the mix loss $-\lambda_t\log\sum_iw_{t,i}e^{-m_{t,i}/\lambda_t}$ lies in $[\underline m_t,\overline m_t]$, because each exponential lies between $e^{-\overline m_t/\lambda_t}$ and $e^{-\underline m_t/\lambda_t}$. Also $\inner{w_t,m_t}\le\overline m_t$, so $\delta_t\le\overline m_t-\underline m_t=r_t$. If $\lambda_t=0$, then $\delta_t=\inner{\pi,m_t}-\underline m_t\le r_t$ as well. In addition, if $\lambda_t>0$, Hoeffding's lemma \citep{hoeffding1963probability}, applied to the random variable $m_{t,I}$ with $I\sim w_t$, which takes values in an interval of length $r_t$, gives
\[
\log\sum_iw_{t,i}e^{-m_{t,i}/\lambda_t}\le-\frac{\inner{w_t,m_t}}{\lambda_t}+\frac{r_t^2}{8\lambda_t^2},
\]
and multiplying by $\lambda_t$ and adding $\inner{w_t,m_t}$,
\begin{equation}
\label{eq:gap-range}
\delta_t\le\frac{r_t^2}{8\lambda_t}.
\end{equation}
Using $\delta_t\le r_t$ and \eqref{eq:gap-range},
\begin{equation}
\label{eq:lambda-square}
\lambda_{t+1}^2-\lambda_t^2=2\lambda_t\delta_t+\delta_t^2\le\frac{r_t^2}{4}+r_t^2=\frac54r_t^2,
\end{equation}
and when $\lambda_t=0$ the same bound follows from $\delta_t\le r_t$ alone. Summing \eqref{eq:lambda-square} over $t$ and using $\lambda_1=0$ gives $\lambda_{T+1}\le\frac{\sqrt5}{2}\bigl(\sum_tr_t^2\bigr)^{1/2}$. Finally, take $u=e_i$ in \eqref{eq:master-telescope}. Then $\inner{w_t-e_i,m_t}=\inner{w_t,m_t}-m_{t,i}$ and $\psi(e_i)=\log(1/\pi_i)$, which gives \eqref{eq:master-bound}.
\end{proof}

%% file: sections/D_lower_det_proof.tex
\section[Proofs of the lower-bound lemmas]{Proofs of Lemma~\ref{lem:diameter} and Lemma~\ref{lem:lower-det}}
\label{app:lower-det}
\subsection[Proof of the diameter coordinate]{Proof of Lemma~\ref{lem:diameter}}
\begin{proof}
The function $(x,y)\mapsto\norm{x-y}$ is continuous on the compact set $\X\times\X$, so it attains its supremum $D$ at some pair $(a,b)$. The midpoint $c$ lies in $\X$ by convexity, and $\norm v=1$ by the choice of $a$ and $b$. Let $x\in\X$. By the Cauchy--Schwarz inequality, $\inner{v,x-a}\le\norm{x-a}\le D$. Since $\inner{v,b-a}=D$,
\[
\inner{v,x-a}=D+\inner{v,x-b}\ge D-\norm{x-b}\ge0.
\]
Thus $\inner{v,x-a}\in[0,D]$, and since $c-a=(D/2)v$ gives $\inner{v,c-a}=D/2$, we obtain $z(x)=\inner{v,x-a}-D/2\in[-D/2,D/2]$. The values $z(a)=-D/2$ and $z(b)=D/2$ follow from $a=c-(D/2)v$ and $b=c+(D/2)v$. The gradient of $\theta z$ is $\theta v$. For $\theta>0$, $\theta z$ is minimized where $z=-D/2$, which holds at $a$; for $\theta<0$ it is minimized where $z=D/2$, which holds at $b$.
\end{proof}

\subsection[Proof of the exact-gradient lower bound]{Proof of Lemma~\ref{lem:lower-det}}
The construction is the blocking argument of \citet{zhang2018adaptive}. The adaptations are the reduction to the coordinate $z$ of Lemma~\ref{lem:diameter}, which makes the construction valid on every compact convex domain, the integers \eqref{eq:lower-blocks}, which handle every real budget $P$, and the moment calculation for $\E|W|$, which gives an explicit constant.

\begin{proof}
Use the blocks $\mathcal I_1,\ldots,\mathcal I_K$ and the padding described after \eqref{eq:rounding}. For every round $t$ in a block, draw an independent sign $\xi_t$, uniform on $\{-1,+1\}$, and let
\[
\ell_t(x)=G\xi_tz(x),\qquad g_t=G\xi_tv.
\]
By Lemma~\ref{lem:diameter}, $\ell_t$ is linear with gradient $G\xi_tv$, so the observation $g_t$ is the exact gradient and the instance satisfies \eqref{eq:oracle} and \eqref{eq:moment-assumption} for every $\sigma\ge0$. The decision $x_t$ is a function of the seed and of $\xi_1,\ldots,\xi_{t-1}$, so it is independent of $\xi_t$, and $\E[G\xi_tz(x_t)]=G\,\E[\xi_t]\,\E[z(x_t)]=0$. The learner's expected cumulative loss is therefore zero.

For block $k$ let $W_k=\sum_{t\in\mathcal I_k}\xi_t$, and let the comparator equal
\[
u^{(k)}=c-\frac D2\sgn(W_k)v\qquad\text{on }\mathcal I_k,\qquad\text{with }\sgn(0)=1.
\]
This comparator lies in $\{a,b\}$ and is constant on each block, so its path length is at most $P$. Its loss on block $k$ is $GW_k\,z(u^{(k)})=-GD|W_k|/2$. Combining this with the zero expected loss of the learner,
\[
\E R_T(u_{1:T})=\frac{GD}2\sum_{k=1}^K\E|W_k|.
\]

It remains to bound $\E|W|$ for a sum $W$ of $L$ independent uniform signs. Expanding the powers and using independence, $\E W^2=L$ and $\E W^4=L+3L(L-1)=3L^2-2L\le3L^2$. H\"older's inequality with exponents $3/2$ and $3$ gives
\[
\E W^2=\E\bigl[|W|^{2/3}|W|^{4/3}\bigr]\le\bigl(\E|W|\bigr)^{2/3}\bigl(\E W^4\bigr)^{1/3},
\]
and raising both sides to the power $3/2$ and rearranging,
\[
\E|W|\ge\frac{(\E W^2)^{3/2}}{(\E W^4)^{1/2}}\ge\frac{L^{3/2}}{\sqrt3\,L}=\sqrt{\frac L3}.
\]
Therefore, using $K\ge\Lambda/2$ and $KL\ge T/2$ from \eqref{eq:rounding},
\[
\E R_T(u_{1:T})\ge\frac{GD}{2\sqrt3}K\sqrt L=\frac{GD}{2\sqrt3}\sqrt{K\cdot KL}\ge\frac{GD}{2\sqrt3}\sqrt{\frac\Lambda2\cdot\frac T2}=\frac{GD}{4\sqrt3}\sqrt{T\Lambda}.
\]
The expectation averages over finitely many sign sequences, and fixing a sequence whose expected regret is at least the average gives fixed losses and a fixed comparator. The second inequality in \eqref{eq:det-lower} follows from $4\sqrt3\le8$.
\end{proof}

%% file: sections/E_anytime_proof.tex
\section[Proof of the unknown-horizon bound]{Proof of Theorem~\ref{thm:anytime}}
\label{app:anytime}
\begin{proof}
Fix $T\ge2$ and a comparator sequence $u_{1:T}$, and let $J$ be the last epoch that starts by round $T$, so that $\tau_J\le T<\tau_{J+1}$. For $j=0,\ldots,J$ let $I_j=\{\tau_j,\ldots,\min(\tau_j+T_j-1,T)\}$; these sets partition $\{1,\ldots,T\}$, and the last one may be shorter than planned. Let $P_j$ be the path length of the comparator inside $I_j$, that is, the sum of $\norm{u_t-u_{t-1}}$ over $t\in I_j\setminus\{\tau_j\}$, and let $v_j=P_j/D$ and $\Lambda_j=1+v_j$. Since regret is a sum over rounds,
\begin{equation}
\label{eq:epoch-regret-sum}
R_T(u_{1:T})=\sum_{j=0}^JR_{I_j},\qquad R_{I_j}=\sum_{t\in I_j}\bigl(\ell_t(x_t)-\ell_t(u_t)\bigr).
\end{equation}
Movements of the comparator between two epochs are not counted in any $P_j$, so
\begin{equation}
\label{eq:epoch-path-budget}
\bar v:=\sum_{j=0}^Jv_j\le\frac{P_T}D=\Lambda_T-1.
\end{equation}

We first bound $\E R_{I_j}$ for a single epoch by applying Theorem~\ref{thm:main} to an auxiliary problem of horizon $T_j$. The auxiliary problem runs over the rounds $\tau_j,\ldots,\tau_j+T_j-1$. On $I_j$ it has the original losses, oracle, and comparator. On the rounds after $I_j$, which exist only if the epoch is cut short by $T$, it has the zero loss with $s_t=0$, an oracle that returns $g_t=0$, and the comparator equal to its value at the last round of $I_j$. The copy of Algorithm~\ref{alg:main} started at round $\tau_j$ makes its decisions on $I_j$ from observations in $I_j$ only, so these decisions are the same in the original and in the auxiliary problem. We take $\F_{\tau_j-1}$ as the initial $\sigma$-field of the auxiliary problem; the copy's decisions are then predictable, \eqref{eq:oracle} and \eqref{eq:moment-assumption} hold on $I_j$ by assumption, and they hold trivially on the added rounds, where $g_t=s_t(x_t)=0$. The comparator of the auxiliary problem is fixed and has path length $P_j$, so its normalized path length is $\Lambda_j$. The proof of Theorem~\ref{thm:main} uses only these properties, so Theorem~\ref{thm:main} applies with horizon $T_j\ge2$. The added rounds have zero loss and contribute zero regret, so the regret of the auxiliary problem equals $R_{I_j}$. Parts (b) and (a) of Theorem~\ref{thm:main} therefore give, respectively,
\begin{align}
\label{eq:epoch-uniform}
\E R_{I_j}&\le3D\bigl(G\sqrt{T_j\Lambda_j}+\sigma T_j^{1/p}\Lambda_j^{a_p}\bigr)+h(\Lambda_j)D\bigl(G\sqrt{T_j}+\sigma T_j^{1/p}\bigr),\\
\label{eq:epoch-logfree}
\E R_{I_j}&\le C_0GD\sqrt{T_j\Lambda_j}+C_p\sigma DT_j^{1/p}\Lambda_j^{a_p}.
\end{align}

Next we bound the sums over epochs that appear when \eqref{eq:epoch-uniform} and \eqref{eq:epoch-logfree} are added. Let $r\in[1/2,1)$. Summing the geometric series,
\begin{equation}
\label{eq:epoch-geometric-sum}
\sum_{j=0}^JT_j^r=\sum_{j=0}^J2^{r(j+1)}=\frac{2^r}{2^r-1}\bigl(2^{r(J+1)}-1\bigr).
\end{equation}
The map $y\mapsto y^r$ is subadditive on $[0,\infty)$, so $x^r\le(x-1)^r+1$ for $x\ge1$. Taking $x=2^{J+1}$ and using $T\ge \tau_J=2^{J+1}-1$ gives
\[
2^{r(J+1)}-1\le\bigl(2^{J+1}-1\bigr)^r\le T^r.
\]
The factor $2^r/(2^r-1)=1/(1-2^{-r})$ is decreasing in $r$, so it is at most its value at $r=1/2$, which is $\sqrt2/(\sqrt2-1)=2+\sqrt2$. Inserting these two bounds into \eqref{eq:epoch-geometric-sum},
\begin{equation}
\label{eq:epoch-geometric}
\sum_{j=0}^JT_j^r\le(2+\sqrt2)\,T^r.
\end{equation}
Every epoch satisfies $\Lambda_j=1+v_j\le1+\bar v\le\Lambda_T$ by \eqref{eq:epoch-path-budget}, so \eqref{eq:epoch-geometric} gives
\begin{equation}
\label{eq:epoch-sum}
\sum_{j=0}^JT_j^r\Lambda_j^{1-r}\le\Lambda_T^{1-r}\sum_{j=0}^JT_j^r\le(2+\sqrt2)\,T^r\Lambda_T^{1-r}.
\end{equation}

\textit{Proof of part (a).} Sum \eqref{eq:epoch-logfree} over $j$ and use \eqref{eq:epoch-regret-sum}. For the first term, \eqref{eq:epoch-sum} with $r=1/2$ gives $GD\sum_j\sqrt{T_j\Lambda_j}\le(2+\sqrt2)A_T(\Lambda_T)$. For the second term, $\Lambda_j^{a_p}=\Lambda_j^{1-1/p}$, and \eqref{eq:epoch-sum} with $r=1/p\in[1/2,1)$ gives $\sigma D\sum_jT_j^{1/p}\Lambda_j^{a_p}\le(2+\sqrt2)B_T(\Lambda_T)$. Hence $\E R_T(u_{1:T})\le(2+\sqrt2)\bigl[C_0A_T(\Lambda_T)+C_pB_T(\Lambda_T)\bigr]$. The decisions of the procedure are $\F_{t-1}$-measurable, because each copy of Algorithm~\ref{alg:main} is and the epoch boundaries are deterministic, so Proposition~\ref{prop:linearization} gives $R_T(u_{1:T})\le C_T$ almost surely, which completes part (a).

\textit{Proof of part (b).} Sum \eqref{eq:epoch-uniform} over $j$. The first group of terms is bounded by $3(2+\sqrt2)\bigl(A_T(\Lambda_T)+B_T(\Lambda_T)\bigr)$ exactly as in part (a). For the second group, $h$ is nondecreasing and $\Lambda_j\le\Lambda_T$, so $h(\Lambda_j)\le h(\Lambda_T)$. By \eqref{eq:epoch-geometric} with $r=1/2$ and with $r=1/p$, $\sum_j\sqrt{T_j}\le(2+\sqrt2)\sqrt T$ and $\sum_jT_j^{1/p}\le(2+\sqrt2)T^{1/p}$. Hence the second group is at most $(2+\sqrt2)h(\Lambda_T)H_T$, and together with $R_T\le C_T$ this proves part (b).

Finally, the procedure stores only the copy for the current epoch. At a round $t$ of epoch $j$ we have $t\ge \tau_j=T_j-1$, so $T_j\le t+1$, and the copy has $\lceil\log_2T_j\rceil+1=O(\log(t+1))$ experts. The cost per round stated in Section~\ref{sec:algorithm} then gives one oracle query and $O(\log(t+1))$ projections.
\end{proof}

%% file: sections/F_high_probability.tex
\section{High-Probability Bounds}
\label{app:high-probability}
This appendix gives a version of Theorem~\ref{thm:main}(a) that holds with probability at least $1-\delta$, and shows that the joint dependence of its noise term on $\Lambda_T$ and $\delta$ cannot be improved under the unconditional moment bound \eqref{eq:moment-assumption}.

\begin{theorem}[High-probability upper bound]
\label{thm:high-probability}
Assume \eqref{eq:oracle} and \eqref{eq:moment-assumption}, and let $u_{1:T}$ be a comparator sequence fixed before the interaction. For every $\delta\in(0,1)$, with probability at least $1-\delta$, Algorithm~\ref{alg:main} satisfies
\begin{equation}
\label{eq:hp-upper}
R_T(u_{1:T})\le\min\Bigl\{C_0A_T(\Lambda_T)+\bigl(C_p+2^{1/p}\bigr)\delta^{-1/p}B_T(\Lambda_T),\ C_T\Bigr\}.
\end{equation}
\end{theorem}

\begin{proof}
Write $\Lambda=\Lambda_T$, $A=A_T(\Lambda)$, and $B=B_T(\Lambda)$, and let $i^\star=\lfloor\log_2(T/\Lambda)\rfloor$ and $L=L_{i^\star}$ as in the proof of Theorem~\ref{thm:main}. By the second part of Proposition~\ref{prop:linearization}, $R_T(u_{1:T})\le C_T$ always, so it suffices to show that the first term in the minimum bounds $R_T(u_{1:T})$ with probability at least $1-\delta$. If $\sigma=0$, then $\epsilon_t=0$ almost surely, and the argument below gives $R_T(u_{1:T})\le C_0A$ almost surely, so we assume $\sigma>0$.

By the subgradient inequality, $\ell_t(x_t)-\ell_t(u_t)\le\inner{s_t(x_t),x_t-u_t}=\inner{g_t,x_t-u_t}-\inner{\epsilon_t,x_t-u_t}$. Summing over $t$,
\begin{equation}
\label{eq:hp-split}
R_T(u_{1:T})\le\sum_{t=1}^T\inner{g_t,x_t-u_t}+\zeta,\qquad \zeta=-\sum_{t=1}^T\inner{\epsilon_t,x_t-u_t}.
\end{equation}
By the second inequality in \eqref{eq:pathwise} with $i=i^\star$, together with $b_{i^\star}\le h(\Lambda)$ from \eqref{eq:prior-cost},
\begin{equation}
\label{eq:hp-pathwise}
\sum_{t=1}^T\inner{g_t,x_t-u_t}\le\sqrt2\sum_{I\in\mathcal B(L)}(D+P_I)\sqrt{V_I}+h(\Lambda)D\sqrt{V_{[T]}}.
\end{equation}
For a set of rounds $I$, let $Z_I=\bigl(\sum_{t\in I}\norm{\epsilon_t}^p\bigr)^{1/p}$. The inequalities \eqref{eq:energy-triangle} and \eqref{eq:norm-comparison} in Appendix~\ref{app:aux} hold for every realization and give $\sqrt{V_I}\le G\sqrt{|I|}+Z_I$. Substituting this into \eqref{eq:hp-pathwise} splits its right side as $\mathrm{Det}+Q$, where
\begin{align*}
\mathrm{Det}&=G\Bigl[\sqrt2\sum_{I\in\mathcal B(L)}(D+P_I)\sqrt{|I|}+h(\Lambda)D\sqrt T\Bigr],\\
Q&=\sqrt2\sum_{I\in\mathcal B(L)}(D+P_I)Z_I+h(\Lambda)DZ_{[T]}.
\end{align*}

We bound the three random quantities $\mathrm{Det}$, $Q$, and $\zeta$ in turn. The term $\mathrm{Det}$ is deterministic. Each block has $|I|\le L$, there are at most $T/L+1$ blocks, and $\sum_IP_I\le D(\Lambda-1)$, so $\sqrt2\sum_I(D+P_I)\sqrt{|I|}\le\sqrt2D\sqrt L(T/L+\Lambda)$, and the proof of Lemma~\ref{lem:scale} shows that $G$ times this quantity is at most $3A$. By \eqref{eq:absorb} with $a=1/2$, $h(\Lambda)GD\sqrt T\le\kappa(1/2)A$. Hence $\mathrm{Det}\le(3+\kappa(1/2))A=C_0A$.

For $Q$, write $\norm{Y}_p=(\E|Y|^p)^{1/p}$ for a random variable $Y$. By \eqref{eq:moment-assumption}, $\norm{Z_I}_p=\bigl(\sum_{t\in I}\E\norm{\epsilon_t}^p\bigr)^{1/p}\le\sigma|I|^{1/p}$ for every deterministic set $I$. The coefficients of the $Z_I$ in $Q$ are deterministic because the comparator is fixed, so Minkowski's inequality gives
\[
\norm{Q}_p\le\sqrt2\sum_{I\in\mathcal B(L)}(D+P_I)\sigma L^{1/p}+h(\Lambda)D\sigma T^{1/p}.
\]
The first term is at most $\sqrt2D\sigma L^{1/p}(T/L+\Lambda)\le3B$ by the same counting and the proof of Lemma~\ref{lem:scale}, and the second is at most $\kappa(a_p)B$ by \eqref{eq:absorb} with $a=a_p$. Hence $\norm{Q}_p\le C_pB$.

For $\zeta$, let $\xi_t=-\inner{\epsilon_t,x_t-u_t}$. The vector $x_t-u_t$ is $\F_{t-1}$-measurable and $\E[\epsilon_t\mid\F_{t-1}]=0$ by \eqref{eq:oracle}, so $\E[\xi_t\mid\F_{t-1}]=0$, and $\E|\xi_t|^p\le D^p\E\norm{\epsilon_t}^p\le D^p\sigma^p$. The inequality of \citet{vonbahr1965inequalities} for sums of martingale differences with $1\le p\le2$ gives
\[
\E|\zeta|^p\le2\sum_{t=1}^T\E|\xi_t|^p\le2TD^p\sigma^p,
\]
so $\norm{\zeta}_p\le2^{1/p}\sigma DT^{1/p}\le2^{1/p}B$, where the last step uses $\Lambda\ge1$.

Combining \eqref{eq:hp-split} with the bound on $\mathrm{Det}$ gives $R_T(u_{1:T})\le C_0A+Y$ with $Y=Q+|\zeta|\ge0$, and Minkowski's inequality gives $\norm Y_p\le(C_p+2^{1/p})B$. Let $\rho=(C_p+2^{1/p})\delta^{-1/p}B$. Markov's inequality applied to $Y^p$ gives
\[
\Pr(Y>\rho)\le\frac{\E Y^p}{\rho^p}\le\frac{(C_p+2^{1/p})^pB^p}{(C_p+2^{1/p})^p\delta^{-1}B^p}=\delta,
\]
so with probability at least $1-\delta$ we have $R_T(u_{1:T})\le C_0A+\rho$, which is \eqref{eq:hp-upper}.
\end{proof}

The next proposition shows that the noise term $\delta^{-1/p}B_T(\Lambda_T)$ in \eqref{eq:hp-upper} cannot be improved for any learner under \eqref{eq:moment-assumption}, jointly in $\Lambda_T$ and $\delta$, up to the ceiling $C_T$ and for each fixed $p$; the constant $C_p+2^{1/p}$ of the upper bound grows as $p\to1$, while the constant $1/240$ of the lower bound does not depend on $p$. It does not by itself give a high-probability minimax theorem for every term. The construction satisfies the unconditional bound \eqref{eq:moment-assumption}, but it need not satisfy the corresponding uniform conditional moment bound: given the latent variable $J$ below, the conditional moment can be large.

\begin{proposition}[High-probability lower bound]
\label{prop:hp-lower}
Let $\X$, $D$, $G>0$, $\sigma>0$, $p\in(1,2]$, and $T\ge2$ be as in Section~\ref{sec:setup}, let $P\in[0,D(T-1)]$ and $\Lambda=1+P/D$, and let $\delta\in(0,1/10]$. For every learner there exist fixed $G$-Lipschitz linear losses, a fixed comparator $u_{1:T}\in\mathcal U(P)$, and an oracle satisfying \eqref{eq:oracle} and \eqref{eq:moment-assumption} such that
\begin{equation}
\label{eq:hp-lower}
\Pr\Bigl(R_T(u_{1:T})\ge\frac1{240}\min\bigl\{GDT,\ \sigma DT^{1/p}\Lambda^{(p-1)/p}\delta^{-1/p}\bigr\}\Bigr)\ge2\delta.
\end{equation}
\end{proposition}

\begin{proof}
Use the coordinate $z$, the points $a$, $b$, $c$, and the vector $v$ of Lemma~\ref{lem:diameter}, the integers $K$ and $L$ of \eqref{eq:lower-blocks}, the blocks $\mathcal I_1,\ldots,\mathcal I_K$, and the padding rounds of Section~\ref{sec:lower}. Let
\[
q=\frac1{2L},\qquad\mu=\min\Bigl\{G,\ \sigma\Bigl(\frac{q^{p-1}}{20\delta}\Bigr)^{1/p}\Bigr\}.
\]
Draw independent signs $\nu_1,\ldots,\nu_K$, each uniform on $\{-1,+1\}$; as in the proof of Theorem~\ref{thm:lower-noise}, each realization defines one instance, and one realization is selected at the end. For every round $t$ in block $\mathcal I_k$ set
\[
\ell_t(x)=\mu\nu_k\,z(x),\qquad u_t=c-\frac D2\nu_kv.
\]
By Lemma~\ref{lem:diameter}, $\ell_t$ is linear with gradient $\mu\nu_kv$, of norm $\mu\le G$, and $u_t\in\{a,b\}$ minimizes $\ell_t$ over $\X$. On the padding rounds the loss is zero and the comparator keeps its last value, so the comparator is constant on each block and its path length is at most $P$.

The oracle draws, independently of the signs and of the learner's seed, a latent variable $J$ with $\Pr(J=1)=10\delta\le1$ and independent Bernoulli variables $\beta_t$ with $\Pr(\beta_t=1)=q$. For $t\in\mathcal I_k$ it returns $g_t=\mu\nu_kv$ if $J=0$ and $g_t=\mu\nu_k(\beta_t/q)v$ if $J=1$, and on the padding rounds it returns $g_t=0$. Fix a realization of the signs. Given $\F_{t-1}$ and $J$, the variable $\beta_t$ is independent of everything observed before round $t$, so $\E[g_t\mid\F_{t-1},J]=\mu\nu_kv$ for both values of $J$; hence $\E[g_t\mid\F_{t-1}]=\mu\nu_kv$, which is the gradient of $\ell_t$, and \eqref{eq:oracle} holds. The noise is $\epsilon_t=0$ if $J=0$ and $\epsilon_t=\mu\nu_k(\beta_t/q-1)v$ if $J=1$. Since $q\le1/2$ and $p>1$, we have $q^{1-p}\ge1$ and
\[
\E|\beta_t/q-1|^p=q\Bigl(\frac1q-1\Bigr)^p+(1-q)\le q^{1-p}+1\le2q^{1-p},
\]
so
\begin{equation}
\label{eq:hp-lower-moment}
\E\norm{\epsilon_t}^p=10\delta\,\mu^p\,\E|\beta_t/q-1|^p\le20\delta\,\mu^pq^{1-p}\le\sigma^p,
\end{equation}
where the last step uses $\mu^p\le\sigma^pq^{p-1}/(20\delta)$. Thus \eqref{eq:moment-assumption} holds for every realization of the signs.

Next we bound the regret. At a round $t\in\mathcal I_k$, $z(u_t)=-\nu_kD/2$ and $|z(x_t)|\le D/2$, so
\begin{equation}
\label{eq:hp-lower-round}
0\le\ell_t(x_t)-\ell_t(u_t)=\mu\Bigl(\nu_kz(x_t)+\frac D2\Bigr)\le\mu D.
\end{equation}
The padding rounds contribute zero, so with $W=\mu DKL$ we have $0\le R_T(u_{1:T})\le W$. Now condition on $J=1$. The signs, the Bernoulli variables, and the seed are independent of $J$, and given $J=1$ the observation on block $k$ is $g_t=\nu_k(\mu/q)\beta_tv$. This is the oracle \eqref{eq:hard-oracle} with $Z$ replaced by $\mu/q$, and the loss \eqref{eq:hard-loss} with the same replacement is $\ell_t$. The argument that led from \eqref{eq:survival-prob} to \eqref{eq:block-noise-lower} uses only these two displays and the independence of the signs from the Bernoulli variables and the seed, so it gives, with $\E_1$ the expectation given $J=1$ and averaging over the signs,
\[
\E_1\sum_{t\in\mathcal I_k}\bigl(\ell_t(x_t)-\ell_t(u_t)\bigr)\ge\frac{\mu D}{2q}\bigl(1-e^{-1/2}\bigr)=\mu DL\bigl(1-e^{-1/2}\bigr)
\]
for every block $k$. Summing over the $K$ blocks,
\begin{equation}
\label{eq:hp-lower-mean}
\E_1R_T(u_{1:T})\ge\bigl(1-e^{-1/2}\bigr)W.
\end{equation}
Since $0\le R_T(u_{1:T})\le W$,
\[
\begin{aligned}
\E_1R_T(u_{1:T})&\le\frac W6\Pr\Bigl(R_T(u_{1:T})<\frac W6\Bigm|J=1\Bigr)+W\Pr\Bigl(R_T(u_{1:T})\ge\frac W6\Bigm|J=1\Bigr)\\
&\le\frac W6+W\Pr\Bigl(R_T(u_{1:T})\ge\frac W6\Bigm|J=1\Bigr).
\end{aligned}
\]
Combining this with \eqref{eq:hp-lower-mean} and with $e^{1/2}\ge1+\frac12+\frac18=\frac{13}8$,
\[
\Pr\Bigl(R_T(u_{1:T})\ge\frac W6\Bigm|J=1\Bigr)\ge1-e^{-1/2}-\frac16\ge1-\frac8{13}-\frac16=\frac{17}{78}\ge\frac15.
\]
Therefore $\Pr(R_T(u_{1:T})\ge W/6)\ge\Pr(J=1)/5=2\delta$, where the probability averages over the signs. This probability is the average over the $2^K$ equally likely sign sequences of the corresponding probability for a fixed sequence, so at least one sign sequence satisfies $\Pr(R_T(u_{1:T})\ge W/6)\ge2\delta$. Fixing that sequence makes the losses and the comparator deterministic, and the oracle satisfies \eqref{eq:oracle} and \eqref{eq:moment-assumption} as shown above.

It remains to bound $W/6$ from below. If $\mu=G$, then $W=GDKL\ge GDT/2$ by \eqref{eq:rounding}, so $W/6\ge GDT/12$. Otherwise $\mu=\sigma\,20^{-1/p}(2L)^{-(p-1)/p}\delta^{-1/p}$, and \eqref{eq:noise-rounding} gives $KL^{1/p}\ge\frac12T^{1/p}\Lambda^{(p-1)/p}$, so
\[
W=\sigma D\,20^{-1/p}2^{-(p-1)/p}\delta^{-1/p}KL^{1/p}\ge\frac12\cdot20^{-1/p}2^{-(p-1)/p}\,\sigma DT^{1/p}\Lambda^{(p-1)/p}\delta^{-1/p}.
\]
Since $20^{-1/p}2^{-(p-1)/p}=\frac12\cdot10^{-1/p}\ge\frac1{20}$ for $p\ge1$, we get $W/6\ge\frac1{240}\sigma DT^{1/p}\Lambda^{(p-1)/p}\delta^{-1/p}$. In both cases $W/6\ge\frac1{240}\min\{GDT,\sigma DT^{1/p}\Lambda^{(p-1)/p}\delta^{-1/p}\}$, which proves \eqref{eq:hp-lower}.
\end{proof}

%% file: sections/G_experiments.tex
\section{Synthetic Experiments}
\label{app:experiments}
This appendix compares Algorithm~\ref{alg:main} with two AdaGrad learners on a one-dimensional problem with Student-$t$ gradient noise.

The domain is $\X=[-1,1]$, so $D=2$, the initial point is $x_0=0$, and the horizon is $T=2^{16}=65{,}536$. The losses are $\ell_t(x)=|x-\theta_t|$, where $\theta_t\in\{-1/2,1/2\}$ starts at $1/2$ and changes sign at the $K$ rounds $\lfloor kT/(K+1)\rfloor+1$, $k=1,\ldots,K$. The comparator $u_t=\theta_t$ has zero loss and path length $K$, so $\Lambda_T=1+K/2$, and the dynamic regret equals $\sum_t|x_t-\theta_t|$. The subgradient selection is $s_t(x)=\sgn(x-\theta_t)$ with $\sgn(0)=0$, so $G=1$. The observation at the point $x$ queried in round $t$ is $g_t=\sgn(x-\theta_t)+0.25\,\xi_t$, where $\xi_1,\xi_2,\ldots$ are independent Student-$t$ variables with $\nu$ degrees of freedom. For $0<p<\nu$,
\[
\E|\xi_t|^p=\frac{\nu^{p/2}\,\Gamma\bigl(\frac{p+1}2\bigr)\Gamma\bigl(\frac{\nu-p}2\bigr)}{\sqrt\pi\,\Gamma\bigl(\frac\nu2\bigr)},
\]
and $\E|\xi_t|^p=\infty$ for $p\ge\nu$. Hence \eqref{eq:moment-assumption} holds for every $p\in(1,2]$ with $p<\nu$, with $\sigma=\sigma_{p,\nu}=0.25(\E|\xi_t|^p)^{1/p}$, and the guarantees of Section~\ref{sec:upper} apply with each such $p$. The parameter $\nu$ of the distribution is therefore not the moment order $p$ of the theorems, and the fixed multiplier $0.25$ does not hold $\sigma_{p,\nu}$ fixed across $\nu$; moreover $\sigma_{p,\nu}\to\infty$ as $p\uparrow\nu$. The variance of the noise is infinite for $\nu\le2$. We use $\nu\in\{1.1,1.3,1.7,3\}$. For $\nu\in\{1.3,1.7,3\}$ the grid of normalized path lengths is $\Lambda_T\in\{1,9,129,513\}$, that is, $K\in\{0,16,256,1024\}$, and for $\nu=1.1$ it is $\Lambda_T\in\{1,3,9,33,129,513\}$, that is, $K\in\{0,4,16,64,256,1024\}$.

We compare three learners. Static is AdaGrad without restarts, with stepsize $D/\sqrt{2V_t}$. The path-budget expert is the restarted expert with block length $L=2^{\lfloor\log_2(T/\Lambda_T)\rfloor}$ used for the upper bound of Theorem~\ref{thm:minimax}; it is supplied the true path length and uses the restart length prescribed by that theorem, not one selected to minimize experimental regret, and it does not use $G$, $\sigma$, or $p$. At $\Lambda_T=1$ its block length is $2^{16}=T$, so it never restarts and coincides with Static. The third learner is Algorithm~\ref{alg:main}, with the cold restarts of Section~\ref{sec:algorithm} and with the warm restarts of Remark~\ref{rem:warm}; a warm restart keeps the expert's current point and resets only its energy $V^{(i)}$ to zero. Algorithm~\ref{alg:main} makes one oracle call per round, at its combined point, and all its experts are updated with that vector. When $\lambda_t=0$ the weights are the prior, as in \eqref{eq:weights}. The weights and the mixability gap are computed in log-sum-exp form after subtracting the largest exponent, and a gap that is negative because of rounding is set to zero. The values $\theta_t$ and the losses are used only to compute the regret; no learner observes them.

Each learner queries the oracle at its own point, and within a run all learners use the same draws $\xi_t$, so the learners can be compared through paired differences. Every setting is repeated in $50$ independent runs. All draws come from one NumPy generator, \texttt{default\_rng(20261119)} (PCG64), which produces at each round one Student-$t$ draw for every pair of setting and run; the study used Python~3.13 and NumPy~2.5.3. For a quantity $X_r$ observed in runs $r=1,\ldots,50$, we report the sample mean $\bar X$, the standard error $s/\sqrt{50}$, where $s$ is the sample standard deviation with denominator $49$, and the half-width $2.0096\,s/\sqrt{50}$ of the $95\%$ confidence interval based on the $t$ distribution with $49$ degrees of freedom. The regret lies in $[0,1.5T]$, so its variance is finite even when the noise has infinite variance. Table~\ref{tab:experiments} reports these statistics, and Figure~\ref{fig:experiments} plots the means with their confidence intervals.

\begin{table}[t]
\caption{Synthetic study, $50$ runs per setting. The third to fifth columns give the sample mean regret $\pm$ standard error. The column $\Delta$ gives the mean of the paired difference between the regret of Algorithm~\ref{alg:main} and that of the path-budget expert, and the last column the mean of the paired difference between Algorithm~\ref{alg:main} with warm and with cold restarts; both are followed by the half-width of a $95\%$ confidence interval. Algorithm~\ref{alg:main} uses cold restarts in the fifth column. Uncertainties below $100$ are given to two significant figures and larger ones to the nearest integer; each value is rounded to the same decimal place as its uncertainty.}
\label{tab:experiments}
\centering
\footnotesize
\setlength{\tabcolsep}{3.5pt}
\begin{tabular}{@{}rrrrrrr@{}}
\toprule
$\nu$ & $\Lambda_T$ & Static & Path-budget & Algorithm~\ref{alg:main} & $\Delta$ & warm $-$ cold \\
\midrule
$1.1$ & $1$ & $6{,}467\pm862$ & $6{,}467\pm862$ & $5{,}674\pm519$ & $-794\pm813$ & $101\pm61$ \\
 & $3$ & $15{,}697\pm1{,}243$ & $10{,}494\pm522$ & $9{,}498\pm624$ & $-995\pm820$ & $36\pm178$ \\
 & $9$ & $23{,}596\pm1{,}048$ & $11{,}262\pm307$ & $10{,}337\pm366$ & $-925\pm561$ & $-59\pm183$ \\
 & $33$ & $31{,}553\pm476$ & $13{,}040\pm143$ & $12{,}485\pm442$ & $-555\pm886$ & $-109\pm148$ \\
 & $129$ & $34{,}374\pm722$ & $15{,}487\pm83$ & $14{,}683\pm488$ & $-805\pm929$ & $-246\pm149$ \\
 & $513$ & $34{,}395\pm200$ & $20{,}103\pm48$ & $17{,}530\pm387$ & $-2{,}573\pm760$ & $-607\pm236$ \\
\midrule
$1.3$ & $1$ & $2{,}073\pm206$ & $2{,}073\pm206$ & $2{,}077\pm148$ & $4\pm184$ & $108\pm47$ \\
 & $9$ & $10{,}486\pm871$ & $5{,}263\pm141$ & $4{,}818\pm188$ & $-445\pm249$ & $-2\pm22$ \\
 & $129$ & $31{,}595\pm137$ & $11{,}197\pm48$ & $9{,}469\pm181$ & $-1{,}729\pm347$ & $-116\pm88$ \\
 & $513$ & $32{,}987\pm73$ & $17{,}165\pm28$ & $13{,}587\pm302$ & $-3{,}578\pm603$ & $-414\pm236$ \\
\midrule
$1.7$ & $1$ & $647\pm24$ & $647\pm24$ & $648\pm24$ & $1\pm14$ & $24\pm11$ \\
 & $9$ & $2{,}789\pm387$ & $2{,}486\pm22$ & $1{,}964\pm22$ & $-523\pm13$ & $-39\pm12$ \\
 & $129$ & $25{,}877\pm303$ & $8{,}291\pm23$ & $6{,}103\pm34$ & $-2{,}188\pm40$ & $-32.8\pm6.1$ \\
 & $513$ & $32{,}335\pm22$ & $14{,}698\pm14$ & $10{,}282\pm32$ & $-4{,}417\pm60$ & $-171.1\pm10.0$ \\
\midrule
$3$ & $1$ & $393.30\pm0.32$ & $393.30\pm0.32$ & $393.43\pm0.28$ & $0.13\pm0.37$ & $0.93\pm0.36$ \\
 & $9$ & $1{,}455.9\pm2.3$ & $1{,}818.2\pm1.5$ & $1{,}446.9\pm1.7$ & $-371.3\pm2.7$ & $-88.9\pm2.6$ \\
 & $129$ & $17{,}232\pm11$ & $6{,}979.7\pm3.7$ & $4{,}966.3\pm3.6$ & $-2{,}013.5\pm6.6$ & $-17.1\pm3.4$ \\
 & $513$ & $31{,}902.2\pm6.6$ & $13{,}240.3\pm5.2$ & $9{,}007.5\pm3.8$ & $-4{,}233\pm10$ & $-166.2\pm5.0$ \\
\bottomrule
\end{tabular}
\end{table}

\begin{figure}[t]
\centering
\input{figures/regret_vs_path}
\caption{Sample mean dynamic regret as a function of $\Lambda_T$, on logarithmic axes, for the four values of $\nu$. Error bars are $95\%$ confidence intervals for the mean and are often smaller than the markers. The gray horizontal line is $T/2=32{,}768$, the regret of playing $x=0$ at every round.}
\label{fig:experiments}
\end{figure}

At $\Lambda_T=1$ the confidence interval for $\Delta$ contains zero for every $\nu$. For $\nu\in\{1.3,1.7,3\}$ and $\Lambda_T\ge9$, the interval lies below zero, so Algorithm~\ref{alg:main} has smaller mean regret than the path-budget expert; for $\nu=1.1$ this holds at $\Lambda_T\in\{3,9,513\}$, and the interval contains zero at $\Lambda_T\in\{1,33,129\}$. The ratio of the mean regret of Algorithm~\ref{alg:main} to that of the path-budget expert ranges from $0.68$ to $1.00$. The combination can give weight to block lengths that fit the realized sequence better than the single length prescribed by the worst-case analysis. For the tested instances with $\nu=1.1$, this ratio ranges from $0.87$ to $0.96$; this describes performance on the tested family and does not establish uniform adaptation as $p\downarrow1$. As $\Lambda_T$ grows, the regret of Static approaches or exceeds $T/2$. Warm restarts have larger mean regret than cold restarts at $\Lambda_T=1$ and smaller mean regret at $\Lambda_T\ge129$, with confidence intervals that exclude zero in all of these settings, and the two mean regrets differ by about $6\%$ or less.

The least-squares slopes of the logarithm of the mean regret of Algorithm~\ref{alg:main} against $\log\Lambda_T$, over the $4$, $3$, $3$, and $3$ values $\Lambda_T\ge9$, are $0.13$, $0.26$, $0.41$, and $0.45$ for $\nu=1.1$, $1.3$, $1.7$, and $3$. For comparison, the exponent $(p-1)/p$ of the noise term tends to $0.09$, $0.23$, $0.41$, and $0.50$ as $p\uparrow\min\{\nu,2\}$. The fitted slopes describe this finite experimental range at one horizon. They do not estimate a minimax exponent or test the guarantee at $p=\nu$, where the moment is infinite; the bound of Theorem~\ref{thm:main} combines a deterministic and a stochastic term and applies with every admissible moment order below $\nu$.

%% file: figures/regret_vs_path.tex
\begin{tikzpicture}
\begin{groupplot}[group style={group size=2 by 2,horizontal sep=1.6cm,vertical sep=1.6cm},
  width=0.46\textwidth,height=0.36\textwidth,xmode=log,ymode=log,log basis x=2,ymax=60000,
  xlabel={$\Lambda_T$},ylabel={regret},tick label style={font=\small},label style={font=\small},title style={font=\small},
  legend style={font=\small,at={(-0.15,-0.4)},anchor=north,legend columns=3}]
\nextgroupplot[title={$\nu=1.1$}]
\addplot+[dashed,mark=square*,error bars/.cd,y dir=both,y explicit] coordinates {(1,6467) +- (0,1731) (3,15697) +- (0,2498) (9,23596) +- (0,2107) (33,31553) +- (0,957) (129,34374) +- (0,1451) (513,34395) +- (0,402)};
\addplot+[dotted,mark=triangle*,error bars/.cd,y dir=both,y explicit] coordinates {(1,6467) +- (0,1731) (3,10494) +- (0,1050) (9,11262) +- (0,618) (33,13040) +- (0,288) (129,15487) +- (0,167) (513,20103) +- (0,96)};
\addplot+[solid,mark=*,error bars/.cd,y dir=both,y explicit] coordinates {(1,5674) +- (0,1043) (3,9498) +- (0,1253) (9,10337) +- (0,736) (33,12485) +- (0,888) (129,14683) +- (0,980) (513,17530) +- (0,777)};
\addplot[gray,thin,forget plot,domain=1:513,samples=2] {32768};
\nextgroupplot[title={$\nu=1.3$}]
\addplot+[dashed,mark=square*,error bars/.cd,y dir=both,y explicit] coordinates {(1,2073) +- (0,413) (9,10486) +- (0,1751) (129,31595) +- (0,275) (513,32987) +- (0,147)};
\addplot+[dotted,mark=triangle*,error bars/.cd,y dir=both,y explicit] coordinates {(1,2073) +- (0,413) (9,5263) +- (0,283) (129,11197) +- (0,97) (513,17165) +- (0,55)};
\addplot+[solid,mark=*,error bars/.cd,y dir=both,y explicit] coordinates {(1,2077) +- (0,297) (9,4818) +- (0,378) (129,9469) +- (0,363) (513,13587) +- (0,606)};
\addplot[gray,thin,forget plot,domain=1:513,samples=2] {32768};
\nextgroupplot[title={$\nu=1.7$}]
\addplot+[dashed,mark=square*,error bars/.cd,y dir=both,y explicit] coordinates {(1,647) +- (0,48) (9,2789) +- (0,778) (129,25877) +- (0,609) (513,32335) +- (0,44)};
\addplot+[dotted,mark=triangle*,error bars/.cd,y dir=both,y explicit] coordinates {(1,647) +- (0,48) (9,2486) +- (0,45) (129,8291) +- (0,46) (513,14698) +- (0,28)};
\addplot+[solid,mark=*,error bars/.cd,y dir=both,y explicit] coordinates {(1,648) +- (0,48) (9,1964) +- (0,45) (129,6103) +- (0,68) (513,10282) +- (0,64)};
\addplot[gray,thin,forget plot,domain=1:513,samples=2] {32768};
\nextgroupplot[title={$\nu=3$}]
\addplot+[dashed,mark=square*,error bars/.cd,y dir=both,y explicit] coordinates {(1,393) +- (0,1) (9,1456) +- (0,5) (129,17232) +- (0,23) (513,31902) +- (0,13)};
\addplot+[dotted,mark=triangle*,error bars/.cd,y dir=both,y explicit] coordinates {(1,393) +- (0,1) (9,1818) +- (0,3) (129,6980) +- (0,8) (513,13240) +- (0,10)};
\addplot+[solid,mark=*,error bars/.cd,y dir=both,y explicit] coordinates {(1,393) +- (0,1) (9,1447) +- (0,3) (129,4966) +- (0,7) (513,9008) +- (0,8)};
\addplot[gray,thin,forget plot,domain=1:513,samples=2] {32768};
\legend{Static,Path-budget expert,Algorithm~\ref{alg:main}}
\end{groupplot}
\end{tikzpicture}